\documentclass{article}
\usepackage[utf8]{inputenc}

\makeatletter
\def\blfootnote{\gdef\@thefnmark{}\@footnotetext}
\makeatother

\usepackage{microtype}
\usepackage{graphicx}
\usepackage{subcaption}
\usepackage{booktabs}
\usepackage{enumitem}
\usepackage{float}
\usepackage{multirow}
\usepackage{multicol}
\usepackage{amsmath}
\usepackage{amssymb}
\usepackage{mathtools}
\usepackage{amsthm}
\usepackage{amsmath}
\usepackage{fullpage}
\usepackage[nice]{nicefrac}
\usepackage{paralist}
\usepackage{algorithm}
\usepackage{algorithmic}
\usepackage{xcolor}
\usepackage{natbib}


\usepackage{hyperref}
\hypersetup{colorlinks,linkcolor=blue}

\usepackage[capitalize,noabbrev,nameinlink]{cleveref}

\theoremstyle{plain}
\newtheorem{theorem}{Theorem}[section]

\newtheorem{lemma}[theorem]{Lemma}

\theoremstyle{definition}
\newtheorem{definition}[theorem]{Definition}

\theoremstyle{remark}
\newtheorem{remark}[theorem]{Remark}

\newcommand{\R}{\mathbb{R}}
\newcommand{\indsize}{\scriptsize}
\newcommand{\colind}[2]{\displaystyle\smash{\mathop{#1}^{\raisebox{.5\normalbaselineskip}{\indsize #2}}}}

\DeclareRobustCommand
\Compactcdots{\mathinner{\cdotp\mkern-2mu\cdotp\mkern-2mu\cdotp}}
\newcommand{\TRUE}{\textbf{true}}
\newcommand{\FALSE}{\textbf{false}}
\newcommand{\TO}{\textbf{to }}
\newcommand{\NOT}{\textbf{not }}

\definecolor{amaranth}{rgb}{0.9, 0.17, 0.31}
\definecolor{green}{HTML}{549D54}

\renewcommand{\algorithmiccomment}[1]{\hfill\textit{\color{green}#1}}
\renewcommand{\algorithmicrequire}{}

\renewcommand{\cite}{\citep}

\title{Simulation of Graph Algorithms with Looped Transformers}
\author{
    Artur Back de Luca$^{*}$\qquad\qquad
    Kimon Fountoulakis$^{*}$\qquad\qquad
}

\date{}

\begin{document}
\maketitle
\def\thefootnote{*}
\footnotetext{David R. Cheriton School of Computer Science, University of Waterloo, Waterloo, Ontario, Canada.}
\def\thefootnote{\arabic{footnote}}
\blfootnote{Emails: \href{mailto:abackdel@uwaterloo.ca}{abackdel@uwaterloo.ca}, \href{mailto:kimon.fountoulakis@uwaterloo.ca}{kimon.fountoulakis@uwaterloo.ca}}

\begin{abstract}
The execution of graph algorithms using neural networks has recently attracted significant interest due to promising empirical progress. This motivates further understanding of how neural networks can replicate reasoning steps with relational data. In this work, we study the ability of transformer networks to simulate algorithms on graphs from a theoretical perspective. The architecture we use is a looped transformer with extra attention heads that interact with the graph. We prove by construction that this architecture can simulate individual algorithms such as Dijkstra’s shortest path, Breadth- and Depth-First Search, and Kosaraju’s strongly connected components, as well as multiple algorithms simultaneously. The number of parameters in the networks does not increase with the input graph size, which implies that the networks can simulate the above algorithms for any graph. Despite this property, we show a limit to simulation in our solution due to finite precision. Finally, we show a Turing Completeness result with constant width when the extra attention heads are utilized.
\end{abstract}

\section{Introduction}

Recent advancements in neural network models have significantly impacted various domains, most notably vision \cite{yuan2021incorporating, khan2022transformers, dehghani2023scaling}, and natural language processing \cite{wei2022emergent, touvron2023llama}. Transformers \cite{vaswani2017attention}, at the forefront of these developments, have become standard for many complex tasks. These successes have also shed light on the capabilities of neural networks in algorithmic reasoning \cite{velivckovic2021neural}, such as basic arithmetic \cite{lee2023teaching}, sorting \cite{tay2020sparse, yan2020neural, rodionov2023neural}, dynamic programming \cite{dudzik2022graph, ibarz2022generalist} and graph algorithms \cite{velivckovic2022clrs, cappart2023combinatorial}.

In this work, we focus on algorithmic reasoning on graphs.
Current empirical results display a promising degree of scale generalization on graphs 
\cite{yan2020neural, tang2020towards, ibarz2022generalist, velivckovic2022clrs, 10.5555/3618408.3618505, numeroso2023dual}. The predominant approach in these studies is to train a neural network to execute a step of a target algorithm and use a looping mechanism to execute the entire algorithm.
The looping mechanism is crucial since it allows the execution of long processes on graphs. Motivated by these empirical results, our goal is to study the ability of looped neural networks to simulate algorithms on graphs of varying sizes. Briefly, by simulation, we mean the ability of a neural network to provide the correct output of a step of an algorithm for every step. In all our results we utilize a looped transformer architecture with extra attention heads that interact with the graph. Instead of storing the graph in the input, we encode it using its adjacency matrix which multiplies the attention head. This allows us to access data from the adjacency matrix without scaling the number of parameters of the network with the size of the graph.

\vspace{1em}
\textbf{Our contributions:} We list our contributions below.
\begin{compactenum}
\item We demonstrate how a looped transformer architecture can simulate graph algorithms like Breadth-first search (BFS) \& Depth-first search (DFS) \cite{moore1959shortest}, Dijkstra’s shortest path algorithm \cite{dijkstra1959note}, and Kosaraju’s algorithm for identifying strongly connected components (SCC) \cite{aho1974design}.
These algorithms were chosen as part of the CLRS benchmark \citep{velivckovic2022clrs}.
We also use these constructions to create a multitask model capable of simulating BFS, DFS, and Dijkstra’s algorithm simultaneously.%
\item In our results, the largest dimension of any weight matrix in our network, denoted by width, does not scale with the number of nodes or edges in the graph. This shows that algorithm simulation is possible for graphs of varying sizes, although limited by finite precision. Current results do not consider the looping mechanism and are constrained by the network’s depth/width \cite{loukas2019graph, Xu2020What}, or they do not consider graph algorithms \cite{perez2021attention, giannou23a}.
\item We also provide a Turing Completeness result for the looped transformer with $O(1)$ width which uses additional attention heads to interact with the graph.
\end{compactenum}

\section{Related work}
\label{sec:related_work}

\textbf{Empirical:} A plethora of papers have been published on the ability of neural networks to execute algorithms on graphs. Notable empirical works include \citet{tang2020towards, yan2020neural, Veličković2020Neural, ibarz2022generalist, 10.5555/3618408.3618505, numeroso2023dual, diao2023relational, anonymous2023beyond, georgiev2023neural, engelmayer2023parallel}. These works leverage looping mechanisms and report favorable scale-generalization results. For a comprehensive review, see \citet{cappart2023combinatorial}.

\textbf{Theoretical:} Currently, three types of complementary results study the ability of neural networks to execute algorithms. The first type is simulation results such as \citet{siegelman95comp, perez2021attention, giannou23a, hertrich2023provably, hertrich2023relu}. These results are demonstrated by providing analytic expressions of the neural networks that achieve simulation, similar to our results. For example, \citet{siegelman95comp, perez2021attention} use simulations to prove the Turing Completeness of Recurrent Neural Networks (RNNs) and Transformers, respectively. \citet{giannou23a} also shows Turing Completeness for Transformers through the simulation of SUBLEQ \cite{mavaddat1988urisc}. However, these works do not consider graph data. Other works, such as \citet{hertrich2023provably, hertrich2023relu} use RNNs and Multilayer-Perceptrons (MLPs) to simulate graph algorithms to solve problems like shortest paths, minimum spanning trees, and maximum flow. Although existing results on graph algorithms or Turing Completeness apply to graph algorithms, these approaches require the graph to be stored in the input data matrix rather than being part of the architecture. This imposes limitations on the size of the input graph, as discussed in \Cref{sec:method}.

The second type of result is using the Probably Approximate Correct learning framework to study the sample complexity of neural networks for executing algorithms.
In \citet{Xu2020What} the authors show that sample complexity is improved when the neural network is aligned with the structure of the algorithm. We also use the concept of alignment in our constructive proofs. 
However, our results are about the simulation of algorithms, which hold for any distribution of the test graph.
Also, \citet{Xu2020What} does not employ a looping mechanism.
The third type is impossibility results. In this case, it is shown that neural networks cannot execute certain algorithms when the depth and width of the network are smaller than a lower bound. In \citet{loukas2019graph} the author studies the ability of graph neural networks to solve various graph problems and also proves Turing Completeness.
However, all results are limited by the depth and the width of the neural network which scales with the size of the input graph. In our case, due to the looping mechanism, the depth and width of the network are constant, and a single neural network can solve a particular problem for any input graph, subject to limitations imposed by finite precision.

\section{Preliminaries}
\label{sec:preliminaries}
This section establishes preliminaries for our task of algorithmic simulation on graphs.
Recognizing that an algorithm consists of multiple steps, we start by defining simulation for an individual step. Consider $h_F: \mathcal{X}\rightarrow \mathcal{Y}$ as the function we aim to simulate and $h_T: \mathcal{X'}\rightarrow \mathcal{Y'}$ as the neural network designed for this purpose.
Note that $h_F$ and $h_T$ might operate in different representation spaces.
For instance, $\mathcal{X}$ and $\mathcal{Y}$ might be the set of natural numbers while $\mathcal{X'}$ and $\mathcal{Y'}$ might be multidimensional reals.
To address this, we use mappings $g_{e}:\mathcal{X}\rightarrow \mathcal{X'}$ and $g_{d}: \mathcal{Y'}\rightarrow \mathcal{Y}$ for encoding and decoding, respectively.
In our constructions in \Cref{sec:input_matrix}, $g_e$ organizes terms into our input matrix format and incorporates biases and positional encodings, while $g_d$ simply extracts the appropriate columns from the output.
\begin{definition}[Simulation]
The neural network $h_T$ is a successful simulator of the algorithmic step $h_F$ if for every input $x \in \mathcal{X}$, $h_F(x) = g_d(h_T(g_e(x)))$. A transformer is said to \emph{simulate} an algorithm if it can simulate each step of the algorithm \citep{giannou23a, perez2021attention}.
\end{definition}
We define a graph $G$ with $n$ nodes by its adjacency matrix $A\in\R_+^{n\times n}$, where $A_{i,j} > 0$ if nodes $i$ and $j$ are connected, otherwise $A_{i, j}$ is zero.
In our work, we also use an input matrix $X\in\R^{K\times d}$, where $d$ is the feature dimension and $K\ge n+1$ is the number of rows that are set according to the simulation.
The structure of $X$ is further explained in \Cref{sec:input_matrix}.
Because of the mismatch of length between $X$ and $A$, we adopt a padded version of $A$, denoted $\Tilde{A}\in\R^{K\times K}$.
The entries of the first row and column of  $\tilde{A}$ are set to zero. This is to align with the input's top row, which holds data not associated with any node, as further discussed in \Cref{sec:input_matrix}. The next $n$ rows and columns correspond to the entries of $A$. Certain implementations require $K > n+1$. In this case, we further pad $\Tilde{A}$ with zeros in the extra rows and columns beyond the entries of $A$. For example, in our Turing Completeness result in \Cref{remark:subleq}, $K$ can exceed $n+1$ because the number of rows may not be directly determined by the number of nodes.

\section{The Architecture}
\label{sec:method}
We utilize a variation of the standard transformer layer in \citet{vaswani2017attention} with an additional attention mechanism that incorporates the adjacency matrix. 
The additional attention mechanism is described by the following equation:
\begin{equation}
\label{eq:attention_head}
\psi^{(i)}(X,\Tilde{A}) := \Tilde{A}\,\sigma\!\left( XW^{(i)}_Q{W^{(i)}_K}^\top X^\top\right)XW^{(i)}_V,
\end{equation}
where $W_V \in \mathbb{R}^{d \times d}$, $W_Q$, and $W_K \in \mathbb{R}^{d \times d_a}$ are the value, query, and key matrices, respectively, $d_a$ represents the embedding dimension and $\sigma$ denotes the hardmax\footnote{The hardmax is defined by [$\sigma(\Phi)]_i:=\sum_{k\in K}e_k/\left|K\right|$, where $e_k$ is the standard basis vector and $K=\left\{k\,|\,\Phi_{ik}=\max(\Phi_{i})\right\}$.}\textsuperscript{,}\footnote{The softmax activation function can be used as well, since softmax approximates hardmax as the temperature parameter goes to zero. However, softmax introduces errors in the calculations. To guarantee small errors during the simulation, the temperature parameter will have to be a function of the size of the graph and other parameters related to finite precision. Softmax could further limit the simulation properties of the architecture beyond what we discuss in \Cref{sec:theory}. However, this type of limitation is common, for example, see \citet{giannou23a, liu2022transformers}. In practice, softmax is preferred over hardmax. For this reason, in our empirical validation results in \Cref{app:sec:empirical_validation} we also use softmax in combination with rounding layers in \Cref{sec:rounding}.} function. Note that setting $\Tilde{A}$ as the identity matrix reduces equation \eqref{eq:attention_head} to a standard attention head.
This modification efficiently extracts specific rows from matrix $A$. Conversely, applying attention to the transpose of $\Tilde{A}$ enables the extraction of columns of $A$. These variations are important in our constructions, and they allow parallel execution of algorithmic subroutines which are further discussed in \Cref{sec:less_than}.
We define a complete transformer layer as:
\begin{equation}
\label{eq:layer}
f(X, \Tilde{A}) = f_\textrm{mlp}(f_\textrm{attn}(X, \Tilde{A}))
\end{equation}
where 
\begin{align*}
f_\textrm{attn}(X, \Tilde{A}) &= X + \sum_{i\in H}\!\psi^{(i)}(X, I_{n+1}) + \sum_{i\in H_A}\!\psi^{(i)}(X, \Tilde{A})
+ \sum_{i\in H_{A^\top}}\!\psi^{(i)}(X, \Tilde{A}^\top)\nonumber
\end{align*}
\begin{equation*}
f_\textrm{mlp}(X) = Z^{(m)}W^{(m)} + X, \ \ Z^{(j+1)} = \phi(Z^{(j)}W^{(j)})
\end{equation*}

for $j=0,\dots,m-1$ and $Z^{(0)} = X$. Here, $X$ represents the input matrix and $\Tilde{A}$ is the padded adjacency matrix. Additionally, $I_{n+1}$ is the identity matrix of size $n+1$, $m$ and $\phi$ are the number of layers and activation function of $f_\text{MLP}$, respectively.
Note that in the formulation of the multi-head attention, we adopt a residual connection and a sum of attention heads $\psi$ in the format described in \eqref{eq:layer}.
$H$, $H_A$, and $H_{A^\top}$ are the index set of attention heads for the standard attention \cite{vaswani2017attention}, as well as for the versions that utilize the adjacency matrix and its transpose, respectively.
The total number of attention heads is $\left |H\right | +\left |H_A\right | + \left |H_{A^\top}\right |$. Furthermore, for the MLP, 
we define $\phi$ to be the ReLU function, we set $m=4$, and we use matrices $W^{(i)} \in \mathbb{R}^{d \times d}$ as the parameters of the MLP, where $d$ is the feature dimension of $X$. We integrate the bias into the linear operations by adding bias columns to $X$ and also introduce a residual connection to $f_{\text{mlp}}$.

\Cref{alg:loop} demonstrates the use of our model $h_T$, which is constructed by stacking multiple layers $f$ as depicted in \eqref{eq:layer}.
In this setting, each forward pass feeds its output to a subsequent pass, creating a looping mechanism.
This loop continues until a pre-defined termination condition is met, such as the activation of a boolean element within the input, denoted by \texttt{term} in \Cref{alg:loop}.
This setting and architecture are used to simulate various graph algorithms, as formally established in the results of \Cref{sec:theory}.
\begin{algorithm}
  \caption{Looped Transformer \cite{giannou23a}}
  \label{alg:loop}
\begin{algorithmic}[1]
\REQUIRE {\bfseries Input:} model $h_T$, matrices $X$ and $\Tilde{A}$, column term
  \WHILE{X[0,term] is \FALSE}
    \STATE $X = h_T(X, \Tilde{A})$
  \ENDWHILE
\end{algorithmic}
\end{algorithm}

\subsection{Discussion on the attention head in \eqref{eq:attention_head}}
While encoding the graph in $X$ is an alternative to our attention head in \eqref{eq:attention_head}, it is subject to limitations. Specifically, concatenating the adjacency matrix to $X$ results in linear dependence of the width of the network to the number of nodes, thereby limiting simulation to graphs with sizes smaller or equal to the one that the network was set up/trained on. Alternatively, appending the edges of the graph to the input results in the number of positional encodings being dependent on the edge count, which may increase quadratically as the number of nodes grows. This can limit simulation performance when the available positional encodings are limited, for instance, due to finite precision. 

Finally, our attention head in \eqref{eq:attention_head} can be viewed as a message-passing layer that performs two convolutions. The first convolution corresponds to a complete graph, i.e., standard attention \citep{vaswani2017attention}, and the second is a standard graph convolution \citep{kipf2016semi}. Multiplying the two convolution matrices is important. In particular, the left multiplication of the attention matrix with $\tilde{A}$ in \eqref{eq:attention_head} allows for direct and easy access to data from $\tilde{A}$. This is because the attention matrix acts as an indicator matrix that extracts rows from $\tilde{A}$, see \Cref{sec:read_a} for details. Note that our attention head in \eqref{eq:attention_head} differs from Graph Attention Networks (GAT) \cite{velivckovic2017graph}, where the graph is incorporated in the attention mechanism. Although possible to use the attention mechanism in GAT to access the graph data, the proofs will be more elaborate than with \eqref{eq:attention_head}. 

\section{Simulation Details}
\label{sec:simulation}
Our simulation results are constructive. We set the parameters of the layer in \eqref{eq:layer} such that we simulate each sub-routine within each algorithm. We start by describing the input to the model in \Cref{sec:input_matrix}. We provide two simulation examples of common subroutines among all algorithms in this paper. In \Cref{sec:less_than} we provide an example of the implementation of the less-than function.
This function is common among algorithms in this paper, and it illustrates how the neural network can achieve parallel computation over the nodes.
Next, in \Cref{sec:read_a} we provide an example for reading information from the graph. Throughout the section, we use Dijkstra's algorithm as an algorithm example (as depicted in \Cref{fig:diagram}).
With the exemption of our Turing Completeness result in \Cref{remark:subleq}, 
it is sufficient to set $K=n+1$ for all algorithms in our paper. For this reason, and for simplicity, we set $K=n+1$ in this section.

\begin{figure*}[t]
    \centering
    \includegraphics[width=0.95\textwidth]{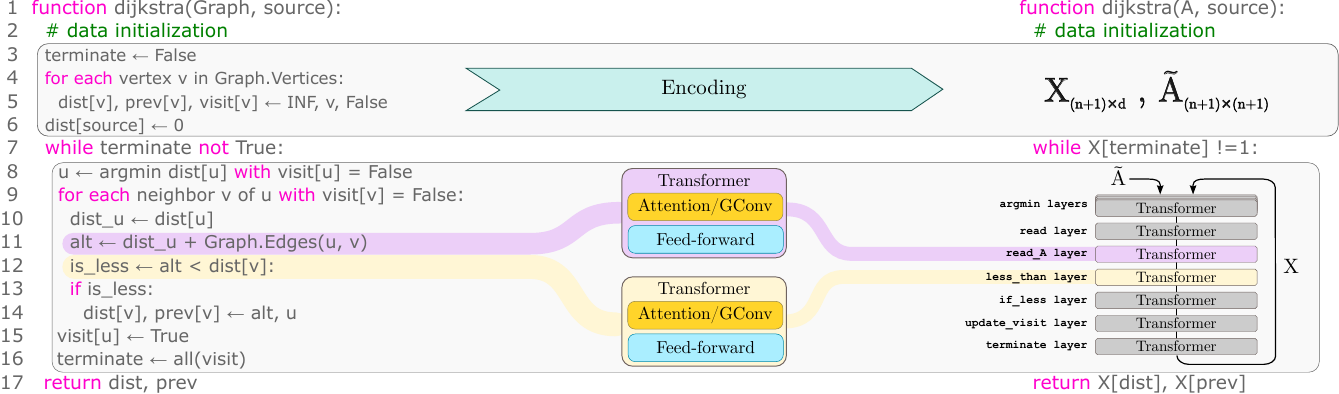}
    \caption{
A simplified illustration of the simulation of Dijkstra's algorithm using a looped transformer with extra attention heads that interact with the graph.
On the left, we display the pseudocode of Dijkstra's algorithm, serving as the source code.
The rightmost section shows the corresponding simulation via a transformer, where each step of the source code in the loop is simulated using one or more transformer blocks.
We specifically focus on lines 11 and 12, highlighted to demonstrate the simulation of individual functions, as discussed in \Cref{sec:simulation}.
At the top center of the figure, the encoding of graph information and variable scopes into $\Tilde{A}$ and $X$ is depicted.
For clarity, $X$ is shown in its transposed format.
Throughout the transformer loop, $\Tilde{A}$ remains constant, while $X$ is updated in each iteration until the simulation meets its termination criteria. Upon termination, the decoding step extracts columns from $X$ that correspond to the algorithm's desired output.}
    \label{fig:diagram}
\end{figure*}

\subsection{Input matrix}
\label{sec:input_matrix}
The input matrix encompasses all variables used in the algorithm. In the case of Dijkstra's algorithm in \Cref{fig:diagram}, the lists of current distances and paths (\texttt{dists} and \texttt{paths}), and variables like the current node and its distance (\texttt{u} and \texttt{dist}) are all incorporated into $X$.
Below, we describe in detail the general structure of $X$. We refer the reader to \Cref{fig:input} for a visualization of the structure of the matrix $X$. 

The top row of $X$ stores single variables, encapsulating the global information relevant to the algorithm.
The bottom $n$ rows of $X$ are reserved for local variables, such as distances and paths, representing the node-specific data. This storage structure separates the global context from the more specific, node-related information of the algorithm. The input matrix $X$ is augmented with columns for positional encodings, biases, and a scratchpad area. Each node receives a unique positional encoding, denoted by $p_i$. These concepts have been present in various contexts, see \citet{giannou23a} for assigning functional blocks to the input and \citet{akyurek2022learning, wei2022statistically, giannou23a} for the notion of scratch space.
Distinct biases are applied to the top and bottom rows to simulate global and local functions effectively. $X$ is equipped with flags indicating global or local states, regulating the algorithm's flow. These flags can either permit or prevent overwriting of fields or signal different stages in the algorithm's execution.
For example, in all our simulations each of the last $n$ entries of the node-wise flag named \texttt{visit} indicates whether the corresponding node has been visited, while the global variable \texttt{term}, signals the terminating condition for the algorithm.
Finally, the scratchpad within $X$ serves as a temporary storage space for variable manipulation or intermediate computations.
\begin{figure}[t]
    \centering
    \includegraphics[width=0.75\columnwidth]{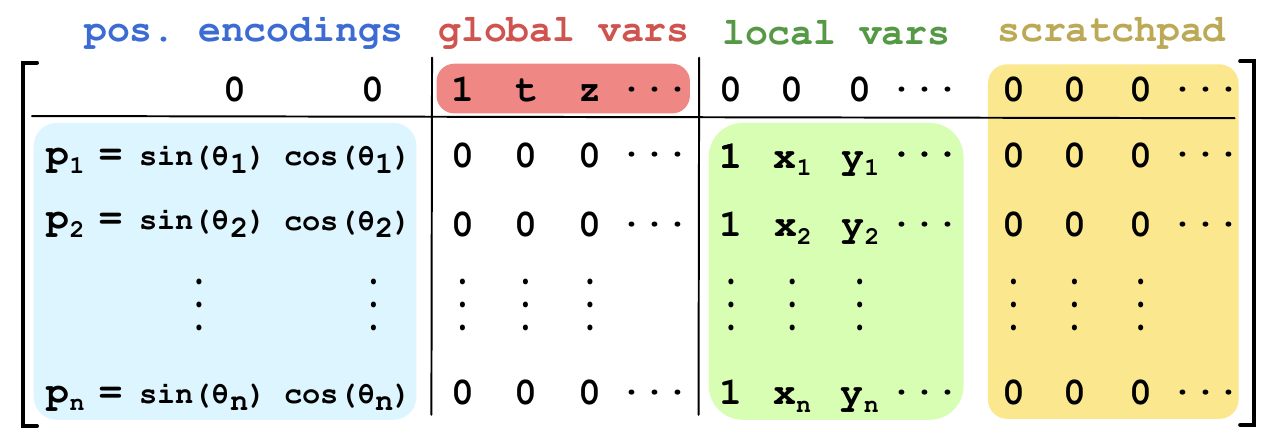}
    \caption{Illustration of the input structure used in the simulation of graph algorithms.
    On the left, columns indicate node positions using circular positional encodings \cite{liu2022transformers}, detailed in \Cref{sec:theory}.
    The next blocks on the right denote the global (in red) and local (in green) variables, which occupy the top and last $n$ rows of $X$, respectively.
    In these blocks, the first column marks the bias of the corresponding variables.
    The symbol $t$ indicates the termination flag, while symbols $z$, and $x_i, y_i$, are generic local and global variables, respectively.
    Finally, depicted in the far right block, the scratchpad is used for temporary storage and calculation.
    Non-shaded areas remain null during execution.
    }
    \label{fig:input}
\end{figure}

\subsection{Less-than}
\label{sec:less_than}
The purpose of this operation is to determine whether each element in a specific column of $X$ is smaller than its counterpart in another column.
The result of this comparison is then recorded in another designated column.
In the context of Dijkstra's algorithm, as shown in \Cref{fig:diagram}, this function plays a critical role in determining if an alternate path is shorter than the shortest path currently known.

We express this function as: \texttt{less-than(X, C, D, E) = write(X[:,C] < X[:,D], X[:,E])}. Here, $X$ is the input matrix, while $C$, $D$, and $E$ are placeholders for the specified column indices in 
$X$. This simulation involves two key steps. First, we approximate the less-than operator ($<$) using the following expression:
\begin{equation}
    \label{eq:less_than}
   X[:,C] < X[:,D] \approx\: \varepsilon^{-1}\phi\left(X[:,D]-X[:,C]\right) -\varepsilon^{-1}\phi\left(X[:,D]-X[:,C]-\varepsilon\right),\nonumber
\end{equation}
where $\varepsilon > 0$ is a small tolerance value, and $\phi$ denotes the ReLU activation function.
This operation approximates a step function of the difference \texttt{X[:,C]-X[:,D]}, where $\varepsilon$ controls the sharpness of the threshold.
If this difference is larger than $\varepsilon$, 
the subtraction of the terms post-ReLU equals $\varepsilon$, which simplifies to $1$.
If the difference is negative, both ReLU terms become zero.
If the difference \texttt{X[:,C]-X[:,D]} is non-negative but less than $\varepsilon$, the approximation linearly interpolates between 0 and 1, with a gradient of $\varepsilon^{-1}$.
The second operation is the \texttt{write} function.
In this operation, the output of the less-than comparison is placed in the specified field, while the original content of that field is erased.
This operation is expressed as:
\texttt{X[:,E]$\,\leftarrow\,$X[:,E]$\,$-$\,$X[:,E]$\,$+$\,$X[:,S$_1$]}, where \texttt{S$_1$} is the scratchpad column that holds the result to be put in \texttt{E}.

Combining these two operations, we define the parameters for the layer $f$.
To understand the purpose behind each operation, we start by defining a particular instance of the matrix $X$, which is used for this example:

\begin{equation*}
X = 
   \mathop{\left[
  \begin{array}{ccccccc}
  \colind{\Compactcdots}{} &
  \colind{1}{$B_\text{global}$} & \colind{0}{$B_\text{local}$} & \colind{c_0}{\raisebox{0.19em}{$C$}}& \colind{d_0}{\raisebox{-0.20em}{$D$}} & \colind{e_0}{\raisebox{0.18em}{$E$}}& \colind{\Compactcdots}{\raisebox{0.18em}{$S$}}\\
  \hline
  \Compactcdots & 0 & 1& c_1& d_1 & e_1 &\Compactcdots \\
  &\vdots & \vdots & \vdots & \vdots & \vdots &\\
  \Compactcdots&0 & 1 & c_n& d_n & e_n &\Compactcdots\\
  \end{array}
  \right].}
\end{equation*}
Here, $C$, $D$ and $E$ are placeholders for column indices used in \eqref{eq:less_than}, $S$ is the scratchpad area, and
$B_\text{global}$ and $B_\text{local}$ represent the biases for global and local variables respectively.
First, in the definition of the attention layer, all values are set to 0, maintaining only the residual connection. 
For the definition of the parameters in $f_{\text{mlp}}$, we introduce the parameters layer by layer, followed by an explanation of what is achieved at each stage.
For the first layer, we have:
\begin{equation*}
(W^{(1)})_{i, j} = \begin{cases}
   1 & \text{if } (i,j\! =\! E)\,\text{or}\,(i=D, j \in \{S_1, S_2\}) \\
   -1 & \text{if } i = C,\; j \in \{S_1, S_2\} \\
   -\varepsilon & \text{if } i\in\{B_\text{global}, B_\text{local}\},\; j = S_2 \\
   0 & \text{otherwise,}
\end{cases}
\end{equation*}
where $S_1$ and $S_2$ are entries in the scratchpad.
Here, the parameters of the first layer $W^{(1)}$ keep the entries of the target column $E$, while it builds the arguments of $\phi$ in \eqref{eq:less_than}, inserting them in the scratchpad.
Storing the intermediate values in the scratchpad is beneficial to prevent overwriting issues, as it avoids the use of columns involved in the definition of the equation. The result is given by:

\begin{equation*}
XW^{(1)} \!=\!{ 
   \mathop{\left[
  \begin{array}{ccccc}
  \colind{\Compactcdots}{} & \colind{e_0}{\raisebox{0.37em}{$E$}}& \colind{d_0 - c_0}{$S_1$} & \colind{d_0-c_0-\varepsilon}{$S_2$} & \colind{\Compactcdots}{}\\
  \hline
  \Compactcdots & e_1 & d_1 - c_1 & d_1 - c_1-\varepsilon & \Compactcdots \\
  & \vdots & \vdots & \vdots &\\
  \Compactcdots & e_n & d_n - c_n & d_n - c_n-\varepsilon & \Compactcdots
  \end{array}
  \right].}}
\end{equation*}
Subsequently, the parameters of the second layer $W^{(2)}$ are:
\begin{equation*}
   (W^{(2)})_{i, j} = \begin{cases}
   1 &\text{if } i=E ,\; j=i\\
   \varepsilon^{-1} &\text{if } i=S_1,\; j=i\\
   -\varepsilon^{-1} &\text{if } i=S_2,\; j=S_1\\
   0 &\text{otherwise.}
\end{cases}
\end{equation*}
Here, $W^{(2)}$ is responsible for subtracting the post-ReLU terms and dividing them by $\varepsilon^{-1}$. This result is then stored in the scratchpad entry $S_1$:

\begin{equation*}
Z^{(1)}W^{(2)} \approx\! 
   \mathop{\left[
  \begin{array}{cccc}
  \colind{\Compactcdots}{} & \colind{e_0}{\raisebox{0.37em}{$E$}}& \colind{c_0 < d_0}{$S_1$} & \colind{\Compactcdots}{}\\
  \hline
  \Compactcdots & e_1 & c_1 < d_1 & \Compactcdots \\
  & \vdots & \vdots &\\
  \Compactcdots & e_n & c_n < d_n & \Compactcdots\\
  \end{array}
  \right].}
\end{equation*}

Finally, for the two remaining operations, we have:
\begin{align*}
(W^{(3)})_{i, j} &= \begin{cases}
   1 &\text{if } i\in\{E, S_1\},\; j=i\\
   0 &\text{otherwise,}
\end{cases}\\
(W^{(4)})_{i, j} &= \begin{cases}
   -1 &\text{if } i,j=E\\
   1 &\text{if } i=S_1,\; j=E\\
   0 &\text{otherwise.}
\end{cases}
\end{align*}

In these two operations, the output of $S_1$ is preserved and then placed in the desired field $E$.
Additionally, the entries of $E$, which were kept in the first three layers, are subtracted from the target field in $W^{(4)}$, effectively removing the previous entry in $E$. Furthermore, to address cases where the input entries are negative, it is necessary to incorporate additional parameters. These parameters should follow the same structure as previously provided, with inverted signs in the parameters $W^{(1)}$ and $W^{(4)}$ for columns $C$, $D$, and $E$.
These additional entries are generally placed in the scratchpad.

\textbf{Parallel computation over the nodes:} The node-wise format in \Cref{sec:input_matrix} is particularly beneficial for operations like less-than in equation \eqref{eq:less_than} because it operates across entire columns.
This approach enables parallel processing for all nodes in a graph, enhancing efficiency in algorithms.
For example, in Dijkstra's algorithm, the typical implementation involves iterating through each non-visited neighbor of a node.
However, by integrating an additional masking function, we can filter out visited or non-neighboring nodes.
This allows us to conduct all necessary comparisons and conditional selections for every node simultaneously within the inner loop of the implementation.
This modification enhances the algorithm's efficiency by reducing the need for repetitive looping in transformer-based simulations.

\subsection{Read row from $A$}
\label{sec:read_a}
In this section, we provide a configuration of \eqref{eq:layer} that extracts a row from the adjacency matrix.
We describe the function as \texttt{read-A(X, Ã, C, D) = write(Ã[C,:], X[:,D])}, where $X$ and $\Tilde{A}$ are the input and padded adjacency matrices, while $C$ and $D$ are the row and column indices for the respective matrices.
The input structure for this operation requires specialized inputs common to all algorithms presented in this work.
This includes a set of columns to store the node-wise positional encodings, denoted as $P$.
Additionally, a global variable representing the positional encoding of the node of interest is stored in a distinct set of columns, marked as $P_\text{cur}$. The structure of the input matrix $X$ is given below:

\begin{equation*}
X = 
   \mathop{\left[
  \begin{array}{cccccccc}
  \colind{\Compactcdots}{}&
  \colind{0}{$P$}& 
  \colind{p_i}{\raisebox{0.19em}{$P_\text{cur}$}}&
  \colind{1}{$B_\text{global}$} & \colind{0}{$B_\text{local}$}&
  \colind{c_0}{\raisebox{0.19em}{$C$}}&
  \colind{d_0}{\raisebox{-0.19em}{$D$}}&
  \colind{\Compactcdots}{\raisebox{0.19em}{$S$}}\\
  \hline
  \Compactcdots & p_1 & 0& 0& 1 & c_1& d_1 &\Compactcdots \\
  &\vdots & \vdots & \vdots & \vdots & \vdots & \vdots &\\
  \Compactcdots&p_n & 0 & 0& 0 & c_n& d_n &\Compactcdots\\
  \end{array}
  \right].}
\end{equation*}
For simplicity, we let the embedding dimension $d_a=2$, which corresponds to the dimension of positional encodings.
For the attention head for $\Tilde{A}$ in \eqref{eq:layer}, we define:
 \begin{align*}
    (W_K, W_Q)_{i, j} &= \begin{cases}
        1 &\text{if } i\in\{(P)_j,(P_\text{cur})_j\},\; j=1,2\\
        0 &\text{otherwise,}
    \end{cases}\\
    (W_V)_{i, j} &= \begin{cases}
       2 &\text{if } i = B_\text{global}, j = D\\
      0 &\text{otherwise.} 
    \end{cases}
 \end{align*}
The multiplication $XW_QW^\top_KX^\top$ results in a matrix where each element is the inner product between positional encodings of nodes. Positional encodings are formulated such that $\left<p_i,p_i\right>>\left<p_i,p_j\right>$ whenever $i \neq j$ for $i, j\in [n]$. By leveraging this property, we allow the attention matrix to behave as an indicator function.
The result is given below:
\begin{equation*}
  \small
  \sigma\!\left(XW_QW^\top_K X^\top\right) =
  {\tiny\left[\begin{array}{c|cccc} \nicefrac{1}{2} & 0& \Compactcdots& \nicefrac{1}{2} & \Compactcdots \\ \hline 0 & 1& \Compactcdots & 0 & \Compactcdots \\ \vdots & \vdots & \ddots & \vdots & \\ \nicefrac{1}{2} & 0& \Compactcdots& \nicefrac{1}{2} & \Compactcdots \\           \vdots & \vdots & & \vdots & \ddots \\ 0 & 0& \Compactcdots & 0 & \Compactcdots\end{array} \right]}.
\end{equation*}

The result of the multiplication of the attention matrix and $\Tilde{A}$ is further multiplied by the input $X$:
\begin{equation*}
\small
\Tilde{A}\,\sigma\!\left(XW_QW^\top_K X^\top\right)X = \nicefrac{1}{2}
  {
  \mathop{\left[
 \begin{array}{ccc}
 \colind{\Compactcdots}{} & \colind{0}{\tiny $B_\text{global}$} & \colind{\Compactcdots}{} \\
 \hline
 \Compactcdots & A_{1i} &\Compactcdots \\
  & \vdots &  \\
 \Compactcdots & A_{in}  & \Compactcdots\\
 \end{array}
 \right].}}
\end{equation*}
The result in column $B_\text{global}$ is written in the target column $D$ by $W_V$.
Notice that in this case, the target field must also be erased.
This can be done as in \Cref{sec:less_than}, using the $f_{\text{mlp}}$ and an intermediate scratchpad field, or by leveraging another attention head to replicate the identity matrix and subtract the values in the target columns.
\section{Theoretical analysis}
\label{sec:theory}
In this section we demonstrate the ability of the architecture in \eqref{eq:layer} to simulate various graph algorithms. We provide empirical validation of the results in \Cref{app:sec:empirical_validation}. We discuss the implications of incorporating attention heads in the form of \eqref{eq:attention_head} and we demonstrate that the overall architecture is Turing Complete, even when constrained to constant width. We start by describing the positional encodings we use and other related parameters in our architecture.

\subsection{Positional encodings and increment}
\label{sec:positional}
In our approach, we use circular positional encodings to enumerate the nodes in the graph \cite{liu2022transformers}.
This enumeration is carried out by discretizing the unit circle into intervals of angle $\hat{\delta}$\footnote{Numerical issues can arise when representing positions by evenly spaced angles around the unit circle.
This is due to the limited precision in representing certain quantities, especially irrational numbers and rational numbers with repeating binary fractions.
To mitigate this issue, we start by setting an initial angle $p_0$ as $(\sin{0}, \cos{0}) = (0, 1)$ and selecting a minimum increment angle, denoted as $\delta$.
We convert the angle $\delta$ into its nearest sine and cosine representation that can be exactly represented by the machine.
We denote this approximate angle as $\hat{\delta}$ and construct the rotation matrix $R_{\hat{\delta}}$. Given that this matrix is derived from $\hat{\delta}$, it consists of quantities that can be represented by the machine. Importantly, we ensure that $R_{\hat{\delta}}$ maintains the fundamental attribute of orthogonality in the rotation matrix.}.
Each node is then represented by a tuple of sine and cosine values. 

For node enumeration, we use the formula $p_i = R^\top_{\hat{\delta}}p_{i-1}$ for $i \ge 1$. In this context, $\hat{\delta}$ represents the smallest rotation that our system can represent with limited precision. This parameter essentially sets the limit on the number of nodes we can enumerate, which is capped at $\lfloor 2\pi\hat{\delta}^{-1} \rfloor$. This limit is determined by the precision with which we can represent $\hat{\delta}$, thus influencing the maximum number of distinct nodes that can be effectively encoded.
This approach inherently compensates for rotational imprecision since the encodings are generated by the rotation matrix and not by the trigonometric functions of each angle. 
Although this could affect the periodicity of the function, our designs do not rely on this attribute. What is crucial is that the increment function can be executed by a neural network, and the positional encodings satisfy the condition: $\left<p_i,p_i\right>>\left<p_i,p_j\right>$ for $i \neq j$ (see \Cref{sec:read_a}). The rotation matrix $R_{\hat{\delta}}$ can be efficiently implemented as a single linear layer within $f_{\text{MLP}}$. Furthermore, since $R_{\hat{\delta}}$ has unitary singular values, all positional encodings have unitary norms, thus 
by the Cauchy-Schwarz Inequality, the inner product property is preserved.

\subsection{Maximum absolute value parameter}
\label{sec:limitation_omega}

All algorithms that we simulate utilize a conditional selection function.
The implementation of this function depends on the parameter $\Omega$, which represents the maximum absolute value within a clause. 
We approximate the conditional selection function in the following manner:
\begin{equation}
   \label{eq:if_else}
   \texttt{if-else($c_1, c_0, \gamma$)} \approx \phi\left(c_0 - \gamma\,\Omega\right) + \phi\left(c_1 - \left(1-\gamma\right)\Omega\right)\nonumber,
\end{equation}

where $c_0, c_1 \in [-\Omega, \Omega]$ are the clause values and $\gamma\in\{0,1\}$ is the condition.
When $\gamma=0$, the function outputs $c_0$, and $c_1$ otherwise.
The complete construction also replicates the terms for the negative counterparts of the clauses.
This approach relies on $\Omega$ to cancel any opposite terms in the conditional selection.
Therefore, $\Omega$ determines the largest absolute value present in the clauses, and any element larger than $\Omega$ breaks the simulation.

\subsection{Results on simulation of graph algorithms}
 
We now present our results on transformer-based simulations of graph algorithms.
We first present the simulation results for Dijkstra's algorithm on weighted graphs, accompanied by a brief overview of its proof.
We then shift our focus to undirected graphs, examining algorithms such as Breadth-First Search, Depth-First Search, and the identification of Strongly Connected Components.
Alongside these results, we provide an outline of the proofs for the related theorems.
Lastly, we present our results on multitasking, which combines the constructions of Dijkstra's algorithm with those of Breadth-First Search and Depth-First Search, and discuss the principles behind this construction.
The complete proofs are in the \Cref{sec:proofs}.

\begin{theorem}
\label{thm:dijkstra}
There exists a looped-transformer $h_T$ in the form of \eqref{eq:layer}, with 17 layers, 3 attention heads, and layer width $O(1)$ that simulates Dijkstra's shortest path algorithm for weighted graphs with rational edge-weights, up to $O(\hat{\delta}^{-1})$ nodes and graph diameter of $O(\Omega\varepsilon)$.
\end{theorem}

\emph{Proof Overview}: \Cref{thm:dijkstra} is proved by construction. 
Our proof begins by adapting Dijkstra's algorithm to fit our looping structure.
This involves unrolling any nested operations to fit within a single loop, as described in \Cref{alg:loop}.
Although these modifications change Dijkstra's algorithm from its traditional form, they do not alter its core functionality.
In conventional programming, loops within the code must be completed before the rest of the code can continue.
However, in the absence of additional structures, our unrolling approach would lead to all instructions being executed all at once.
This could potentially disrupt the intended sequence of operations.
To address this, we also implement binary flags and conditional selections, as explained in \eqref{eq:if_else}.
These mechanisms are vital for ensuring that certain operations only modify data under specific conditions.
A key example in our simulations is the minimum function which has $O(n)$ complexity and runs concurrently with the rest of the algorithm.
When this function is active, it prevents modifications by other parts of the algorithm.
The modifications can resume once the minimum function has been completed, at which point it becomes inactive.
Finally, each step of the rewritten algorithm is implemented as a series of transformer layers in the form of \eqref{eq:layer}, accounting for 17 layers, and three attention heads: two standard and one specialized over $A$.
Some of the core simulation mechanisms of these functions are extensively discussed in \Cref{sec:method} and \Cref{sec:theory}, demonstrating that most of the constructions are based on previously described principles. Furthermore, \Cref{thm:dijkstra} is restricted to rational edge weights due to the inherent limitations of finite precision in representing edge values. We reweight the edges to circumvent issues arising from the tolerance $\varepsilon$ of the less-than function \eqref{eq:less_than}, which may not adequately capture minor path differences.
This involves dividing all edge weights by the smallest absolute edge weight. Consequently, this ensures that the minimal non-zero path difference is effectively greater than one. The graph size is bounded by the enumeration constraints of $\hat{\delta}$, as mentioned in \Cref{sec:positional}. The dependence on the graph diameter is attributed to the maximum distance that can be constructed during the execution of Dijkstra's algorithm.
Due to the algorithm's design, which excludes repeating or negative paths, the quantity is determined by the weighted graph's diameter. After the reweighting strategy, the diameter must remain below $\Omega$, leading to a dependence of $O(\Omega\epsilon)$ on the original graph diameter.

We now present our results for algorithms on unweighted graphs, followed by a collective discussion of their proofs.

\begin{theorem}
\label{thm:bfs}
There exists a looped-transformer $h_T$ in the form of \eqref{eq:layer}, with 17 layers, 3 attention heads, and layer width $O(1)$ that simulates Breadth-First Search for unweighted graphs with up to $\min(O(\hat{\delta}^{-1}),\,O(\Omega))$ nodes.
\end{theorem}

\begin{theorem}
\label{thm:dfs}
There exists a looped-transformer $h_T$ in the form of \eqref{eq:layer}, with 15 layers, 3 attention heads, and layer width $O(1)$ that simulates Depth-First Search for unweighted graphs with up to $\min(O(\hat{\delta}^{-1}),\, O(\Omega))$ nodes.
\end{theorem}

\begin{theorem}
\label{thm:scc}
There exists a looped-transformer $h_T$ in the form of \eqref{eq:layer}, with 22 layers, 4 attention heads, and layer width $O(1)$ that simulates Kosaraju's Strongly Connected Components algorithm for unweighted graphs with up to 
$\min(O(\hat{\delta}^{-1}),\,O(\Omega))$ nodes.
\end{theorem}

\emph{Proof Overview:} Our proofs for Theorems \ref{thm:bfs}-\ref{thm:scc} also follow a constructive approach.
The implementations of breadth-first and depth-first search are comparable to Dijkstra's algorithm in terms of the number of layers and attention heads. However, the Strongly Connected Components algorithm utilizes significantly more layers and an additional attention head for $A$, due to its complexity and the need for a depth-first search over both $A$ and $A^\top$.
Despite the absence of edge weights, we encounter an additional bound of $O(\Omega)$ on the number of nodes.
This stems from our method of replicating the behavior of queues and stacks, which are used in these algorithms. Similar to our approach with Dijkstra's algorithm, we also incorporate an interaction with a minimum function.
During the execution of the algorithms, priority values are assigned to the neighbors of each node visited.
These assigned values shape the minimum function's behavior: increasing priority values mimic the functionality of a queue while decreasing values replicate that of a stack.
Hence, having more nodes in a graph increases the number of loop iterations, thereby increasing the absolute value of the priority value. 
Given that these priority values feature in conditional selections, they are also subject to the constraints imposed by $\Omega$. Consequently, the graph's size limitations are independently determined by either $\hat{\delta}$ or $\Omega$.
In the implementation of Breadth-First Search (BFS) and Depth-First Search (DFS), since nodes only need to be visited once,
$\Omega$ is the maximum node count 
for which accurate simulation of any graph is guaranteed.
In the implementation of the Strongly Connected Components (SCC) algorithm, which involves two distinct DFS operations, the maximum node count is constrained to $\Omega/3$. This specific limit arises from the two-phase process, and also the role of the first DFS in organizing nodes by their finishing times, which is further detailed in \Cref{sec:proofs}.

We now present our results on multitasking, followed by a discussion of its construction.
The following remark encapsulates the capabilities of our unified model.

\begin{remark}
\label{remark:multitask}
There exists a looped transformer $h_T$ in the form of \eqref{eq:layer}, with 19 layers, 3 attention heads and layer width $O(1)$ that simulates
(i) Depth-First Search and Breadth-First Search for unweighted graphs with up to $\min(O(\hat{\delta}^{-1}), O(\Omega))$ nodes; and (ii) Dijkstra's shortest path algorithm for weighted graphs with rational edge-weights with up to $O(\hat{\delta}^{-1})$ and graph diameter of $O(\Omega\varepsilon)$.
\end{remark}

Our unified model can simulate any of the three algorithms with the right input configuration. Although the input structure remains constant across all three algorithms, it only needs minor modifications to specify the algorithm to be executed. This result leverages the construction components of the previous models and thus inherits the limitations presented by the previous results. These algorithms share common structures and functions, which are reused to avoid redundancy, as well as distinct components specific to each algorithm, such as variables and functions that need to be accommodated.
For instance, all algorithms use the minimum function as well as a common termination criterion. 
The accommodation of unique functions for each algorithm is achieved through a conditional selection function.
This function determines the variables to be updated, ensuring the execution reflects the intended algorithmic behavior.

Overall, while our guarantees are limited by parameters $\varepsilon$, $\Omega$, and $\hat{\delta}$, an important question remains: What are the implications of using a graph that fails to meet assumptions in our theorems? This situation could result in inaccurate simulation outcomes or even a failure to meet the pre-established termination condition, thus failing to halt as expected.

\subsection{Turing Completeness}

We now present our result on Turing Completeness of the architecture in \eqref{eq:layer}.
It is important to note that using a graph convolution operation within the attention mechanism as shown in \eqref{eq:attention_head} might affect the model's expressiveness. We demonstrate that a transformer based on this structure is also Turing Complete. Following the methodology of \citet{giannou23a}, we demonstrate this claim by successfully simulating SUBLEQ, a single-instruction language that, with arbitrary memory resources, is proven to be Turing Complete \cite{mavaddat1988urisc}.

\begin{remark}
\label{remark:subleq}
There exists a looped-transformer $h_T$ in the form of \eqref{eq:layer}, which utilizes the modified attention head in \eqref{eq:attention_head}, with 11 layers, 3 attention heads, and layer width $O(1)$ that simulates SUBLEQ.
\end{remark}

\emph{Proof Overview:} We prove that the architecture in \eqref{eq:layer} simulates a modified version of SUBLEQ, which is also Turing Complete. This modification changes the memory structure of SUBLEQ, utilizing two memory blocks instead of one. 
The first block is the standard memory block as in \citet{mavaddat1988urisc} which can be read and altered.
The second is a special read-only memory that only stores the adjacency matrix.
We introduce this memory modification to reflect the fact that the architecture in \eqref{eq:layer} has two inputs, the data matrix $X$ and the padded adjacency matrix $\tilde{A}$.
The former is modified at each iteration of the looped transformer architecture, while the latter is only used in the convolution stage, and it never changes.
The modified version of SUBLEQ utilizes data from the graph by reading the data from the second block and copying them to the first block.
We show that this operation of reading and copying data can be simulated by our read operation in \Cref{sec:read_a}.
The rest of SUBLEQ can be simulated using generic attention heads and MLPs similarly to \citet{giannou23a}.
The complete proof is provided in \Cref{sec:proofs}.

\section{Training Limitations in Algorithm Simulation}

In this section, we address the inherent challenges of recovering the constructive parameters that simulate the aforementioned algorithms through training.
Despite demonstrating the existence of parameters capable of simulation, discovering them through gradient-based training is challenging.

This difficulty arises from the need to approximate discontinuous functions -- such as conditional selection or the less-than function -- with continuous functions.
In such cases, using neural networks to capture the sharp transitions of discontinuities leads to severe ill-conditioning. Ill-conditioning manifests as a sharp and narrow region of the loss function around the desired solution, which hampers parameter recovery during training.
Moreover, as empirically demonstrated in \Cref{sec:appdx:ill_conditioning}, the finer the simulation parameters (e.g., the smaller $\varepsilon$ is in \Cref{eq:less_than}), the stronger the ill-conditioning, thus making recovery increasingly difficult.

However, these findings do not undermine the potential of neural networks to solve such tasks.
It simply highlights the issues related to simulating specific algorithms using neural networks. 
Alternatives to this problem are further discussed in \cite{hertrich2023relu}, where the authors present algorithms for minimum spanning tree and maximum flow that circumvent the discontinuities associated with more conventional approaches like Kruskal's and Edmonds-Karp's algorithms, thus avoiding approximation issues.

We believe that, with current training methods, the solutions that demonstrate good generalization capabilities may also indicate the existence of algorithms that can effectively solve these tasks without the use of conditional branching, which is primarily responsible for the observed discontinuities.

\section{Conclusion and Future Work}
\label{sec:limitations_future}
We present a constructive approach where a looped transformer, combined with graph convolution operations, is used to simulate various graph algorithms. This demonstrates the potential of looped transformers for algorithmic reasoning on graphs. This architecture also proves to be efficient, as it has constant network width regardless of the input graph's size, allowing it to handle graphs of varying dimensions.

Outside the scope of simulation, studying looped transformers for graph algorithms within the Probably Approximate Correct (PAC) learning framework offers an exciting research direction.
Specifically, the constructive parameters $\Omega$, $\hat{\delta}$, and $\varepsilon$ are essential in determining which graphs can be simulated.
However, these parameters may also influence the learnability of algorithmic subroutines.
An interesting research question involves determining the sample complexity of subroutines as a function of these constructive parameters.
Investigating this could further reveal the relationship between the design of solutions and their learning effectiveness under the PAC framework.

\section*{Acknowledgments}
K.~Fountoulakis would like to acknowledge the support of the Natural Sciences and Engineering Research Council of Canada (NSERC). Cette recherche a \'et\'e financ\'ee par le Conseil de recherches en sciences naturelles et en g\'enie du Canada (CRSNG), [RGPIN-2019-04067, DGECR-2019-00147].

\bibliographystyle{plainnat}
\bibliography{references}

\newpage
\appendix
\onecolumn
\section{Additional related work}
\label{app:sec:related_work}
A substantial body of research has focused on testing the capability of neural networks to execute algorithms. Early efforts include models like Hopfield Networks \cite{hopfield1985neural}, which were applied to combinatorial optimization tasks \cite{smith1999neural}. Prior to the emergence of transformers, models such as Neural Turing Machines (NTMs) \cite{graves2014neural} and Memory Networks \cite{weston2014memory} integrated neural networks with external memory, showing strong performance in tasks like copying and sorting. Subsequent advances were seen by adapting NTMs for learning algorithmic tasks \cite{kaiser2015neural}, incorporating a dynamic memory \cite{graves2016hybrid}, and utilizing recurrent models to learn multiple programs \cite{reed2015neural}. These advancements laid the groundwork for concepts used in transformers and this work. For instance, the addressing mechanism in \citet{graves2014neural} is akin to the attention mechanism in transformers, as well as concepts like scratch space for intermediate computations, as proposed by \citet{reed2015neural}.

In the more recent developments related to algorithmic reasoning in transformers,
dedicated frameworks have emerged to represent and execute complex algorithmic tasks.
An example of this is the Restricted Access Sequence Processing Language (RASP).
RASP allows for the expression of sequences of steps in transformers and can be applied to tasks such as token counting and sorting.
Along this line, \citet{lindner2023tracr} developed a compiler for translating between RASP and Transformer models,
while \citet{friedman2023learning} have designed modified Transformers that can be trained and subsequently converted to RASP instructions.

\section{Implementation details}
\label{sec:implementation}
In this section, we describe aspects of the implementation and empirical verification of our simulations.
As mentioned in \Cref{sec:method}, despite our constructions being based on hardmax,
for the purpose of empirical validation and to better align with the implementation practices in transformers, we utilize the softmax function:

\begin{equation*}
   \sigma_S(z_i) = \frac{e^{z_i}}{\sum_{j=1}^{K} e^{z_j}},
\end{equation*}

where $K$ represents the number of elements.
When using the softmax in attention, we divide all parameters in $W_Q$ and $W_K$ by a temperature parameter $T$.
The value of $T$ is important as it determines the sharpness of the softmax operation.
As shown in \citet{giannou23a} and \citet{liu2022transformers}, as the temperature gets arbitrarily small, the behavior of the softmax approximates that of the hardmax, which is used in our constructions.

However, it is important to acknowledge that in practice the simulation of all algorithmic steps is subject to numerical limitations due to floating-point arithmetic.
For example, using the minimum angle $\hat{\delta}$ that a computer can represent accurately is impractical. The reason is that this quantity becomes overshadowed by subsequent operations that introduce imprecision, such as matrix multiplications, and by non-linear elements like the softmax function in the attention stage.
In algorithmic tasks, simulations exhibit a very low tolerance for computational errors. If not properly managed, these errors can propagate, leading to inaccurate prediction or non-termination.
To address this, we adopt a more conservative approach in setting parameters for $\hat{\delta}$ and $\varepsilon$ as described in \Cref{app:sec:empirical_validation}. 
In addition, we incorporate rounding operations for both binary values and positional encodings, following the structure described in \eqref{eq:layer}.

\subsection{Rounding functions}
\label{sec:rounding}
We utilize rounding functions in-between the simulated algorithmic steps to ensure that the computation of binary values made by the transformer does not include additional errors that propagate along the simulation.
To this end, we adapt the rounding operation presented by \citet{giannou23a}, which can be algebraically expressed as:
\begin{equation}
    \label{eq:round_appdx}
    \text{round}(X[:,C])\approx \eta^{-1}\left(\phi(X[:,C]-\nicefrac{1}{2})-\phi(X[:,C]-(\nicefrac{1}{2}+\eta))\right),
\end{equation}
where $\eta$ is a small tolerance value. This implementation can also be done using the structure presented in \eqref{eq:layer}, noting that the expression in \eqref{eq:round_appdx} is equivalent to the \texttt{less-than} function presented in \Cref{sec:less_than}, with $\varepsilon=\eta$, and $X[:,D]$ equivalent to the quantity $1/2$.

In practice, the increment function described in \Cref{sec:increment_appdx} introduces noise in the positional encodings.
To counter this, we employ a rounding procedure for circular positional encodings.
This involves using two transformer layers.
The first layer generates a binary vector based on the current positional encoding
We then round this quantity utilizing the procedure described in \eqref{eq:round_appdx}, and finally obtain the corresponding positional encoding based on the index position using the last layer.

The first attention layer is set up as described in \eqref{eq:attn_read}, where $D$ is set to $B_\text{global}$, and $S_\text{PE}$ denotes the target column in the scratchpad.
Another attention head is used to clear the original scratchpad contents, as detailed in \eqref{eq:read_clear}. The corresponding $f_\text{MLP}$ in this layer implements \Cref{eq:round_appdx}, with $S_\text{PE}$ as the source column $C$.

In the second transformer layer, the matrices $W^{(1)}_K$, $W^{(1)}_Q$, and $W^{(1)}_V$ of the first attention head are defined as follows:

\begin{align*}
(W^{(1)}_K, W^{(1)}_Q)_{i, j} = \begin{cases}
    \Tilde{T}^{\nicefrac{-1}{2}} &\text{if } i=S_\text{PE}\\
    0 &\text{otherwise,}
\end{cases}\quad
(W^{(1)}_V)_{i, j} = \begin{cases}
   2 &\text{if } i=(P)_k, j=(E)_k,\; k\in{0,1}\\
  0 &\text{otherwise.} 
\end{cases}
\end{align*}

In this layer, $E$ denotes the target columns for inserting positional encodings, and $\Tilde{T}$ is an annealed temperature parameter of softmax, which is set higher than in other implementations (e.g., $T=10^{-5}$).
This ensures that minor differences in this attention head are minimized, thereby reducing errors in reading positional encodings. Additionally, we use the configuration of \eqref{eq:read_clear} in another attention head to clear the contents of $E$ and $S_\text{PE}$. Lastly, the implementation of $f_\text{MLP}$, as defined in \eqref{eq:read_mlp}, clears the last $n$ entries of the target field $E$.

\subsection{Writing prevention after termination}
The constructions in the empirical section slightly differ from the ones presented in \Cref{sec:proofs} due to an additional consideration.
While the theoretical constructions require delivering accurate outputs upon reaching the termination condition, one could continue running the looped transformer after the termination flag is activated. To ensure that the solution is not modified post-termination, we introduce extra constraints in the form of flags that deactivate write permissions once the termination condition is met.
These are articulated through algebraic expressions, easily integrating into the transformer implementation.
In practice, the writing permission flag is not determined solely by the minimum function's termination flag (\texttt{term$_\text{min}$}) for the implementation of \Cref{thm:dijkstra} and similar algorithms. Instead, we construct a flag defined as \texttt{write = term$_\text{min}$ and (not term)}, which can be implemented as $\texttt{X[:,write]} = \phi(\texttt{X[:,term$_\text{min}$]-X[:,term]})$.
This condition ensures that once the algorithm reaches the termination state, writing is forbidden.

Furthermore, we employ a conditional selection mechanism -- as explained in \Cref{sec:if_else_appdx} -- to maintain the termination flag, which is activated if the terminating condition is met but the simulation continues running.
This is performed by first moving the calculation of the termination condition to a temporary storage area, denoted as $S_\text{term}$. 
Then, at the end of each iteration, we execute the layer: \texttt{cond-select(X, S$_\text{term}$, term, S$_\text{term}$, term)}.
This means that if the recalculated termination condition in $S_\text{term}$ is true, then the \texttt{term} variable is updated; otherwise, it remains unchanged, effectively managing the looped transformer's operation post-termination.
\subsection{Empirical validation}
\label{app:sec:empirical_validation}
We validate our theoretical results using graphs in the CLRS Algorithmic Reasoning Benchmark \cite{velivckovic2022clrs}.
This benchmark offers data created for multiple tasks, including sorting, searching, and geometry, among others.
Additionally, it encompasses tasks specific to graphs, including those simulated in this study.
These tasks are divided into segments with varying input sizes. For graph-related tasks, the training and validation sets include graphs with 16 nodes, while the test set contains graphs with 64 nodes.
In \Cref{tab:experiments}, we show that our simulations achieve perfect accuracy across all tested instances,

\begin{table}[h]
\centering
\caption{Accuracy over algorithmic tasks on CLRS}
\label{tab:experiments}
\begin{tabular}{lcccc}
\hline
Split (size) & BFS & DFS & Dijkstra & SCC \\ \hline
Train (1000) & 100\% & 100\% & 100\% & 100\% \\
Validation (32) & 100\% & 100\% & 100\% & 100\% \\
Test (32) & 100\% & 100\% & 100\% & 100\% \\ \hline
\end{tabular}
\end{table}

where we also observe the same performance for the multitask model. In our empirical verifications, we consistently apply specific parameters: an angular increment $\delta$ set at $10^{-2}$, a maximum value $\Omega$ of $10^5$, and a temperature parameter $T$ at $10^{-7}$.
Additionally, in the masking operations such as \Cref{sec:mask_visited_appdx}, we address numerical imprecision during execution by using a value that is an order of magnitude smaller than $\Omega$.
This precaution is crucial because minor inaccuracies that slightly increase these values above $\Omega$ could lead to cumulative errors in the conditional selection function, potentially impacting the integrity of our results.

Furthermore, beyond the practical considerations discussed in the previous sections, the implementation of breadth-first search in the CLRS benchmark slightly differs from the one presented in \Cref{alg:bfs}.
This is because, as mentioned in \citet{velivckovic2022clrs}, the results of the algorithm are built on top of the concept of parallel execution, which implies that for algorithms such as breadth-first search, there is no concept of order in the discovery of neighbor nodes, as outlined in \citet{cormen2022introduction}.
Instead, the order of discovery in CLRS is established by the value of the node index.
Therefore, to address this particularity, the implementation of the conditional selection function in step (9) of \Cref{alg:bfs} has another pair of clauses, which can be expressed by \texttt{cond-select(X, val$_\text{best}$, order, term$_\text{min}$, order)}.
Essentially, this variation in the implementation is designed to capture the priority value from the parent node, thereby ensuring uniform priority assignment across all nodes within the same neighborhood level.

\subsection{Ill-conditioning effects on recovering simulation parameters}
\label{sec:appdx:ill_conditioning}
In this section, we discuss the challenges that hinder the discovery of simulation parameters for any of the algorithms presented through training.
We begin by examining an important distinction in simulation results: exact versus approximate representation.
When analyzing the expressivity of neural networks, most results focus on their ability to approximate desired target functions \cite{cybenko1989approximation,hornik1989multi,lin2017does, zhou2020universality,you2020are}.
However, only certain functions can be \emph{exactly} represented by neural networks, i.e., with no approximation error.
In \cite{arora2018understanding}, the authors demonstrate that feedforward neural networks with ReLU activation functions are inherently limited to exactly representing continuous and piecewise-linear functions.
This limitation does not prevent such neural networks from representing other functions, albeit only approximately.
This issue extends beyond the specific architecture studied in \cite{arora2018understanding}, as it can be argued that any neural network using continuous functions can only approximate discontinuous functions.
Our simulation results directly observe this limitation: the construction of discontinuous functions, such as if-else or less-than, are subject to a degree of approximation controlled by the parameters $\Omega$ and $\varepsilon$, as shown in \Cref{eq:if_else} and \Cref{eq:less_than}, respectively.

Relying on approximations for discontinuous functions introduces a significant challenge: ill-conditioning.
In discontinuous functions, small input increments can cause sharp changes in outputs.
Capturing this in the linear operations of neural networks, which must approximate these sharp variations, leads to ill-conditioning.
The more precise the approximation, the more ill-conditioned the solution becomes, directly affecting the discovery of such parameters through training.
Even with strong supervision \cite{velivckovic2022clrs, Xu2020What} in the form of ground-truth inputs and outputs for every layer, recovering such approximations is challenging.

\begin{figure}[ht]
    \centering
\includegraphics[width=.65\textwidth]{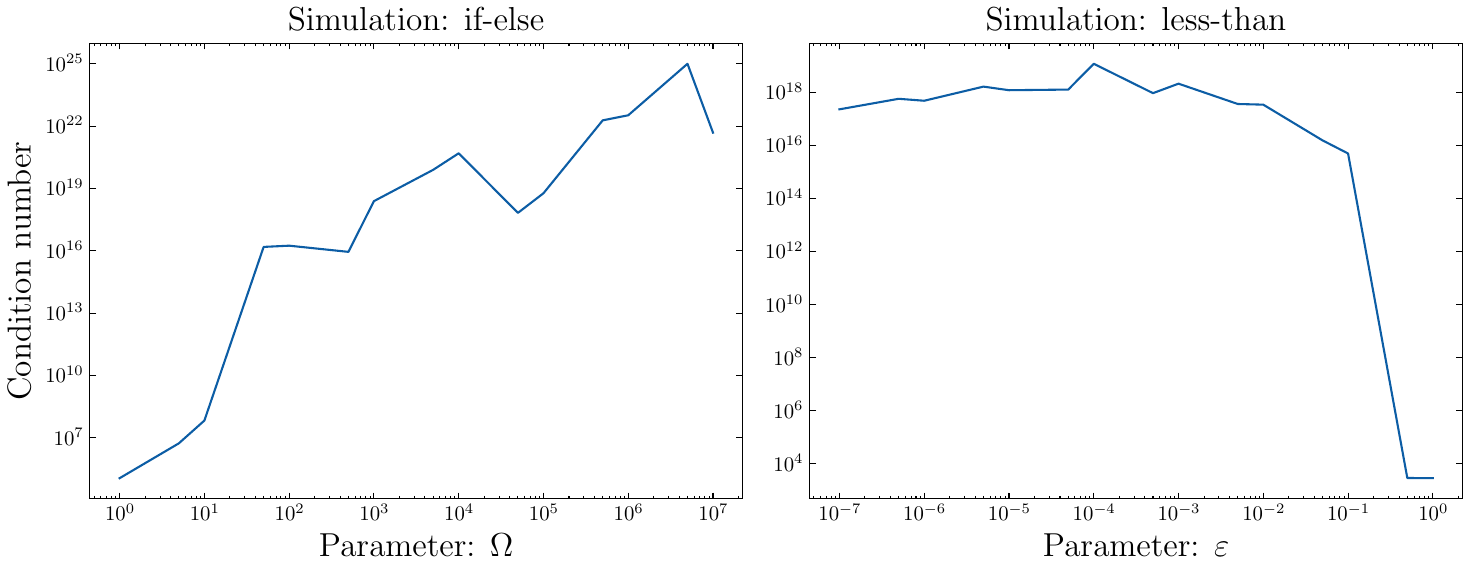}
    \caption{Condition number of linear layers in the constructions of the if-else (left) and less-than (right) functions. The x-axes indicate the value of the construction parameters $\Omega$ and $\varepsilon$ of \Cref{eq:if_else} and \Cref{eq:less_than}, respectively. A higher $\Omega$ improves the simulation quality for the if-else function, while a smaller $\varepsilon$ improves the simulation quality for the less-than function.}
    \label{fig:appendix_kappa}
\end{figure}

To illustrate this, we use our construction of the if-else function (\Cref{eq:if_else}), controlled by the parameter $\Omega$, which limits the highest absolute value in any clause, determining the range of the approximation.
We then estimate the Hessian with respect to the loss function as a function of $\Omega$, as shown in \Cref{fig:appendix_kappa}.
Analyzing the condition number of this Hessian provides insight into the sharpness of the region surrounding the optimal solution: a higher condition number indicates a sharper and narrower region around the local optimum, which complicates recovery during training.

Our empirical results show that once $\Omega$ exceeds 50, the condition number of the Hessian rises above $10^{16}$. This indicates an optimization landscape so sharp that finding the solution through any learning routine becomes virtually impossible.
To further support this claim, we conducted an additional experiment using the if-else function.
We initialized the layer parameters with the construction values, adding a very small quantity of random noise (of magnitude $10^{-7}$).
Despite conducting numerous experiments, initializing the parameters very close to the optimal target, and utilizing various optimization strategies—including second-order methods—the parameters failed to converge to the target solution.

This experiment further illustrates that any modern architecture relying on continuous functions to approximate discontinuities will struggle to discover parameters through training that achieve perfect (or even good) scale generalization. The better the scale generalization required, the more severe the ill-conditioning becomes, making it increasingly difficult to discover good parameters.

\clearpage
\section{Proofs of Theorems}
\label{sec:proofs}
\subsection{Preparations for the proofs}
\label{sec:min_appdx}
In this section, we first introduce key coding primitives that are essential to all the algorithmic simulations presented.
We illustrate these coding primitives through the algorithm of the minimum function.
This function is designed to identify the smallest value in a list, along with its index and other pertinent information.
As outlined in \Cref{sec:theory}, for each simulated algorithm, we adapt their formulations to align with the looped architecture of the transformer.
Although this architecture inherently does not support inner loops, we simplify the presentation of the pseudocode by using inner loops solely to depict functions that are operated in parallel by the transformer.
As further depicted in Algorithms \ref{alg:Dijkstra}-\ref{alg:scc}, the iteration of the minimum function described in \Cref{alg:mininum} is executed inside the main loop of the transformer, along with the other functions of the corresponding algorithms.
The algorithm description is also accompanied by comments that enumerate the steps that are implemented using the layer formulation in \eqref{eq:layer}.
Each comment is also linked to the section that provides a more detailed explanation of its implementation.

{\centering
\begin{minipage}{.7\linewidth}
\begin{algorithm}[H]
    \small
  \caption{Iteration of the minimum function (get\_minimum)}
  \label{alg:mininum}
\begin{algorithmic}[1]
  \REQUIRE {\bfseries Input:} data $x$, size $n$
  \COMMENT{List of values to find the minimum}
  \REQUIRE {\bfseries Input:} {\color{gray}{data $x_{\text{SCC}}$, size $n$}}
   \\\hrulefill
  \STATE {\color{gray}{scc$_\text{cur}$, scc$_\text{best}$ = 0, 0}}
  \COMMENT{Only used for SCC}
  \STATE idx$_\text{cur}$, val$_\text{cur}$ = 0, 0
  \STATE idx$_\text{best}$, val$_\text{best}$ = 0, $\Omega$
  \STATE visit$_\text{min}$ = array of size $n$
\FOR{$i = 1$ \TO $n$}
    \STATE visit$_\text{min}$[i] = \FALSE
\ENDFOR
  \STATE term$_\text{min}$ = \FALSE
   \\\hrulefill
    \IF{term$_\text{min}$ is \TRUE}
      \STATE $\cdots$
    \STATE term$_\text{min}$ = \FALSE
    \COMMENT{(1) Re-initialize variables }[\ref{sec:if_else_appdx}]
    \ENDIF
    \STATE idx$_\text{cur}$ = idx$_\text{cur}$ + 1
    \COMMENT{(2) Increment position }[\ref{sec:increment_appdx}]
    \STATE val$_\text{cur}$ = $x$[idx$_\text{cur}$]
    \COMMENT{(3) Read value at current position }[\ref{sec:read_appdx}]
    \STATE {\color{gray}{scc$_\text{cur}$ = $x_{\text {SCC}}$[idx$_\text{cur}$]}}
    \STATE condition = val$_\text{cur}$ $<$ val$_\text{best}$
    \COMMENT{(4) Compare values }[\ref{sec:less_than_appdx}]
    \IF{condition}
        \STATE val$_\text{best}$, idx$_\text{best}$ = val$_\text{cur}$, idx$_\text{cur}$
        \COMMENT{(5) Update variables }[\ref{sec:if_else_appdx}]
        \STATE {\color{gray}{scc$_\text{best}$ = scc$_\text{cur}$}}
    \ENDIF
    \STATE visit$_\text{min}$[idx$_\text{cur}$] = \TRUE
    \COMMENT{(6) Visit current position }[\ref{sec:visit_min_appdx}]
    \STATE term$_\text{min}$ = \NOT (\FALSE \text{ in visit$_\text{min}$}) \COMMENT{(7) Trigger termination }[\ref{sec:terminate_min_appdx}]
\end{algorithmic}
\end{algorithm}
\end{minipage}
\par
}

The lines [1-8] consist of the initialization of the data, while the remaining represent the algorithm executed inside the loop.
Note that lines 1, 15, and 19 are highlighted in gray, and are only defined in the case of the Strongly Connected Components algorithm.
As described in \Cref{sec:input_matrix}, the single variables (e.g. \texttt{val$_\text{cur}$}) occupy the top row, while lists (e.g. \texttt{visit$_\text{min}$}) occupy the last $n$ rows of $X$.
Furthermore, the implementation of indexes requires two different representations in the input matrix $X$.

In \Cref{alg:mininum}, these correspond to \texttt{idx$_\text{cur}$} and \texttt{idx$_\text{best}$} as well as the variables for the Strongly Connected Component implementation.
In $X$, these indices are represented both as integer numbers and positional encodings. The latter is vital for algorithm functionality, while integer numbers are key in the decoding phase, forming the output data.
Consequently, ambiguous variables in proofs are referred to using their positional encoding.
For example, if \texttt{idx$_\text{cur}$} points to the list's first element during execution, it is represented as $p_1=(\sin(\theta_1), \cos(\theta_1))$ — the positional encoding for the first node. The integer number representation is emphasized when relevant.
Finally, the comments that enumerate and describe the steps executed are the basis for the functions implemented by our architecture.

\subsubsection{Conditional selection: Steps (1) and (5)}
\label{sec:if_else_appdx}
This section details the implementation of the conditional selection functions outlined in steps (1) and (5) of \Cref{alg:mininum}.
We reference the structure of this implementation in the subsequent proofs.
In our constructions, we express this operation as: \texttt{cond-select(X,V$_\text{1}$,V$_\text{0}$,C,E) = write(if-else(X[:,V$_\text{1}$], X[:,V$_\text{0}$], X[:,C]), X[:,E])}. Here, $X$ is the input matrix; $V_1$ and $V_0$ are columns for clause values; $C$ and $E$ are the condition and target field, respectively.
The approximation of the function \texttt{if-else} is shown in \eqref{eq:if_else}.
Next, we describe how this function is implemented. As with other implementations, we utilize a scratchpad area, denoted as $S$, to store intermediate calculations. Each column in this space is indexed by $S_i$.
During implementation, we also introduce more abstract indices for the scratchpad, such as $S_{\text{idx$_\text{cur}$}}$.
While it is understood that this index correlates to a specific integer number, we adopt this approach 
to enhance clarity and comprehension. By using such descriptive indexing, we aim to convey the intended purpose of each allocation in the scratchpad.

As also described in the implementation of the \texttt{less-than} function (see \Cref{sec:less_than}), we set all parameters to zero in the attention stage, only maintaining the residual connection.
Then, for the values of the $f_{\text{MLP}}$, we define:
\begin{align}
\label{eq:if_else_appdx}
(W^{(1)})_{i, j} &= \begin{cases}
   1 & \text{if } i\in\left\{\substack{{V_0, V_1, C, E,}\\{B_\text{global}, B_\text{local}}}\right\},\; j=i, \\
   0 & \text{otherwise,}
\end{cases}
\;
&&(W^{(2)})_{i, j} = \begin{cases}
   1 & \text{if } (i,j)\in \left\{\substack{(E, E), (V_0, S_1)\\(V_1, S_2)}\right\}\\
   -\Omega & \text{if } (i,j)\in \left\{\substack{(C, S_1), (B_\text{global}, S_2)\\(B_\text{local}, S_2)}\right\}\\
   \Omega & \text{if } i = C, j = S_2\\
   0 & \text{otherwise,}\nonumber
\end{cases}\\
(W^{(3)})_{i, j} &= \begin{cases}
   1 &\text{if } (i,j)\in \left\{\substack{(E, E), (S_1, S_1)\\(S_2, S_1)}\right\}\\
   0 & \text{otherwise,}
\end{cases}
\;
&&(W^{(4)})_{i, j} = \begin{cases}
   1 & \text{if } i = E, j = S_1 \\
   -1 & \text{if } i, j = E\\
   0 & \text{otherwise.}
\end{cases}
\end{align}

In the general construction of the conditional selection function, the first layer, $W^{(1)}$, functions as an identity operator. The second layer, $W^{(2)}$, constructs the arguments of $\phi$ as detailed in \eqref{eq:if_else}.
The third layer, $W^{(3)}$, sums the terms post-ReLU, and the last layer, $W^{(4)}$ is designed to transfer the term, initially located in the scratchpad, to the specified column $E$.
Our complete construction also accounts for the negative counterparts of the clauses.
For simplicity, the adaptation for these counterparts is not included in the proofs, as it simply involves inverting the signs of the inputs $V_0, V_1$ and $E$ in $W^{(1)}$ and $W^{(4)}$.
While we present an implementation for a single pair of clauses $V_0$ and $V_1$, this function can accommodate additional variables following the same mechanism as in \eqref{eq:if_else}, as is done for steps (1) and (5) of Algorithm \ref{alg:mininum}.
For example, in the implementation of \Cref{alg:mininum}, the conditional selection clause represented in step (5) also needs to update the integer representation of the index, therefore another pair of clauses is utilized in (5).

Essentially, the general implementation of the conditional selection function operates by choosing a value from two columns in the input matrix and inserting it into a designated field. However, the implementation varies in the context of the re-initialization function. In this scenario, the conditional selection is executed between the current values in the algorithm and their reset versions. Specifically for this function, the reset values must be generated in the first layer and then placed in the scratchpad, which assumes the role of the $V_1$ field in the aforementioned formulation. This adjustment ensures that the function appropriately handles the re-initialization process within the algorithmic framework.
\begin{equation}
(W^{(1)})_{i, j} = \begin{cases}
   1 & \text{if } i = \left\{B_\text{global}, B_\text{local}, \text{idx$_\text{cur}$, val$_\text{cur}$, ..., visit$_\text{min}$}\right\}, j=i\\
   (p_0)_k & \text{if } i=(\text{idx$_\text{cur}$})_k, j=(S_{\text {idx$_\text{cur}$}})_k,\, k\in \{1, 2\}\\
   \Omega & \text{if } i=\text{val$_\text{best}$}, j=S_{\text{val$_\text{best}$}}\\
   0 &\text{if } i=\text{val$_\text{cur}$, ..., visit$_\text{min}$}, j=S_{\text{zero}}\\
   0 & \text{otherwise.}
\end{cases}
\end{equation}

In this implementation, the variable for the current index is reset to a specific value, $p_0 = (\sin(0), \cos(0)) = (0, 1)$. This value does not correspond to any node enumerated and is intentionally reserved in our constructions to represent the equivalent of the index $0$ in integer numbers. This approach ensures that $p_0$ serves as a unique and identifiable starting point or reset state within the algorithm.
For the remaining layers, the placeholder of field $V_1$ is then occupied by the targets of $W^{(1)}$, i.e. $S_{\text{idx$_\text{cur}$}}, S_{\text{val$_\text{best}$}}$ and $S_{\text{zero}}$ in their corresponding operations.
Moreover, it is important to note that regardless of the established condition, the status of term$_\text{min}$ is invariably set to false. Consequently, this field is consistently cleared in every iteration, similar to the way the target field $E$ is cleared.

\subsubsection{Increment position: step (2)}
\label{sec:increment_appdx}
After the re-initialization process, the first step that follows is obtaining the next index in the list.
In this case, since we utilize circular positional encodings, all that is needed is to utilize a rotation matrix.
The rotation matrix is predetermined by $\hat{\delta}$, as discussed in \Cref{sec:positional}.
We express this operation as \texttt{increment(X, C, D) = write(rotate(X[:, C]), X[:,D])}, where $C$ and $D$ are the source and target columns.
In the implementation of step (2), both $C$ and $D$ are the variable \texttt{idx$_\text{cur}$}.
Below, we detail the specific set of parameters used to implement this function.
Again, in this case, we only utilize the entries of $f_{\text{MLP}}$, while the parameters of $f_\text{attn}$ are set to zero.
\begin{align*}
(W^{(1)})_{i, j} &= \begin{cases}
    \cos(\hat{\delta}) &\text{if } i=(C)_k,\; j=(S_{C})_k,\; k\in\{1,2\}\\
    -\sin(\hat{\delta}) &\text{if } i = (C)_1,\; j=(S_{C})_2\\
    \sin(\hat{\delta}) &\text{if } i = (C)_2,\; j=(S_{C})_1\\
   1 &\text{if } i\in D,\; j=i\\
   0 & \text{otherwise,}
\end{cases}\\
(W^{(2,3)})_{i, j} &= \begin{cases}
1 &\text{if } i\in \left\{D, S_C\right\},\; j=i\\
0 & \text{otherwise,}
\end{cases}\quad
(W^{(4)})_{i, j} = \begin{cases}
1 &\text{if } i=(S_C)_k, j=(D)_k, k\in\{1,2\}\\
-1 &\text{if } i\in D, j=i\\
0 &\text{otherwise.}
\end{cases}
\end{align*}

In this context, the initial entries of the matrix $W^{(1)}$ correspond to the rotation matrix $R_{\hat{\delta}}$.
Additionally, this matrix includes a unique non-zero entry, specifically designed to reset the target field of the final layer.
It is also important to note that the variables $C$ and $D$ are represented by a tuple of sine and cosine values, and as such, its coordinates must be accurately aligned.
Moreover, the implementation must account for the fact that positional encodings can possess negative values. The parameters for these added entries are essentially the inverse of the signs present in $W^{(1)}$ and $W^{(4)}$. To facilitate this process, the implementation also incorporates additional scratchpad entries, which serve to temporarily hold these intermediate results.

\subsubsection{Read value: Step (3)}
\label{sec:read_appdx}
Once the node index is updated, we need to retrieve its associated value from the input list, possibly accompanied by other relevant information. 
This operation can be expressed as \texttt{read-X(X, C, D, E) = write(X[C, D], X[1,E])}, where $C$ and $D$ denote the coordinates of the target entry and $E$ is the placeholder for the target field.
In a more basic form, this operation involves extracting information from a local variable and transferring it into a global variable.
For the implementation of step (3), we utilize this function to extract the entry from the list of values (denoted by $x$ in \Cref{alg:mininum}) as well as from the list that encodes the indexes with integer numbers.

The principle for this operation is similar to the extraction of rows or columns of A, as presented in \Cref{sec:read_a}.
We employ a standard attention mechanism (no graph convolution) leveraging the positional encoding of the idx$_\text{cur}$ variable to obtain the entries in the corresponding row.
 \begin{align}
  \label{eq:attn_read}
    (W^{(1)}_K, W^{(1)}_Q)_{i, j} = \begin{cases}
        1 &\text{if } i\in\{(P)_j, (C)_j\},\; j=1,2\\
        0 &\text{otherwise,}
    \end{cases}\quad
    (W^{(1)}_V)_{i, j} = \begin{cases}
       2 &\text{if } i=D, j=E\\
      0 &\text{otherwise.} 
    \end{cases}
 \end{align}

To ensure the target column $E$ is prepared for the insertion of the desired value, we apply a distinct attention mechanism.
This mechanism employs a standard attention head but sets the first entry of $XW_K$ and $XW_Q$ to $p_0$. This configuration guarantees that, when executing a hardmax operation, the attention matrix effectively functions as an identity matrix. Consequently, this facilitates the removal of any existing entries in the target column $E$.

The matrices $W^{(2)}_K$, $W^{(2)}_Q$, and $W^{(2)}_V$ are defined as follows:
 \begin{align}
 \label{eq:read_clear}
    (W^{(2)}_K, W^{(2)}_Q)_{i, j} = \begin{cases}
        1 &\text{if } (i,j)\in\left\{((P)_1, 1), ((P)_2, 2), (B_\text{global}, 2)\right\}\\
        0 &\text{otherwise,}
    \end{cases}\quad
    (W^{(2)}_V)_{i, j} = \begin{cases}
       -1 &\text{if } i, j = E\\
      0 &\text{otherwise.} 
    \end{cases}
 \end{align}

This results in the following attention matrix:
 \begin{equation*}
 \small
  \sigma\!\left(XW^{(2)}_QW^{(2)\top}_K X^\top\right) = { \sigma\left(\left[\begin{array}{c|cccccc} p_0^\top p_0 & p_0^\top p_1& \Compactcdots& p_0^\top p_n \\ 
  \hline p_1^\top p_0 & p_1^\top p_1& \Compactcdots & p_1^\top p_n\\ 
  \vdots & \vdots & \ddots & \vdots \\ p_n^\top p_0 & p_n^\top p_1& \Compactcdots & p_n^\top p_n\end{array}\right]\right)} =
  {\left[\begin{array}{c|ccccc} 1 & 0& \Compactcdots& 0 \\
  \hline 0 & 1& \Compactcdots & 0 \\
  \vdots & \vdots & \ddots & \vdots \\ 
  0& 0& \Compactcdots &  1\end{array} \right]}  
\end{equation*}

By employing the value matrix $W_V$, we can clear the target column $E$ to insert new values.
It is important to note that the attention head in \eqref{eq:attn_read} also adds terms to the last $n$ rows of $E$, which are undesirable.
To counteract this, we use the parameters of the $f_\text{MLP}$:
\begin{align}
\label{eq:read_mlp}
(W^{(1)})_{i, j} = \begin{cases}
   1 & \text{if } i,j=E\\
   -\Omega &\text{if } i=B_\text{global}, j=E\\
   0 & \text{otherwise,}
\end{cases}
\quad
(W^{(2,3)})_{i, j} = \begin{cases}
   1 & \text{if } i,j=E\\
   0 & \text{otherwise,}
\end{cases}
\quad
(W^{(4)})_{i, j} = \begin{cases}
   -1 & \text{if } i,j=E\\
   0 & \text{otherwise,}
\end{cases}
\end{align} 

In this construction, the first layer removes the top row entry of $E$, which is then passed along the second and third layers, and finally subtracted from the target, removing all entries in the bottom rows.

\subsubsection{Compare values: step (4)}
\label{sec:less_than_appdx}
In this stage, we focus on comparing the current value (\texttt{val$_\text{cur}$}) with the best value found so far (\texttt{val$_\text{best}$}). The implementation of this function is detailed in \Cref{sec:less_than}. In this context, the operation is expressed as \texttt{less-than(X, val$_\text{cur}$, val$_\text{best}$, condition)}.

\subsubsection{Visit node: step (6)}
\label{sec:visit_min_appdx}
Each iteration of the minimum function sets the current node as visited. This action is crucial for the function's termination criteria, which checks if all nodes have been visited.

Essentially, the role of this function is to assign a predefined value to a specific entry $(X)_{i,j}$. We express this operation as \texttt{write-row(X, C, D, E) = write(X[0,D].e$_{X[0,C]}$, X[:,E])}. Here, $e_i$ symbolizes the standard basis vector in an $n+1$ dimensional space. The variables $C$, $D$, and $E$ in the expression represent columns that contain the row to be written, the writing content, and the target field, respectively. This writing procedure is analogous to the read function, previously described in \Cref{sec:read_appdx}.

The matrices $W_K$, $W_Q$, and $W_V$ are defined as follows:
\begin{align}
\label{eq:attn_visited}
(W^{(1)}_K, W^{(1)}_Q)_{i, j} = \begin{cases}
    1 &\text{if } i\in\{(P)_j, (C)_j\},\; j=1,2\\
    0 &\text{otherwise,}
\end{cases}\quad
(W^{(1)}_V)_{i, j} = \begin{cases}
   2 &\text{if } i=D, j=E\\
  0 &\text{otherwise.} 
\end{cases}
\end{align}

In the application of step (6) of \Cref{alg:mininum}, we set $C$ to be idx$_\text{cur}$, $D$ to $B_\text{global}$, and $E$ to the visit$_\text{min}$ variable.
Additionally, as also pointed out in \Cref{sec:read_appdx}, the extra entries generated by the attention head in \eqref{eq:attn_visited} must also be eliminated.
Contrary to the read function, in this instance, the unwanted entry appears in the top row of the target field.
To address this, we adopt the same strategy as in the read function, utilizing an additional attention head that functions as the identity function, deleting the top row entry using the global bias $B_\text{global}$.
 \begin{align*}
    (W^{(2)}_K, W^{(2)}_Q)_{i, j} = \begin{cases}
        1 &\text{if } (i,j)\in\left\{((P)_1, 1), ((P)_2, 2), (B_\text{global}, 2)\right\}\\
        0 &\text{otherwise,}
    \end{cases}\quad
    (W^{(2)}_V)_{i, j} = \begin{cases}
       -1 &\text{if } i = B_\text{global},\; j = E\\
      0 &\text{otherwise.} 
    \end{cases}
 \end{align*}

To define $f_\text{MLP}$, we set all its parameters to zero and retain only the residual connection. This approach essentially disregards any contribution from $f_\text{MLP}$.

\subsubsection{Terminate: step (7)}
\label{sec:terminate_min_appdx}
At the end of each step of the minimum function, as well as in the other algorithms, we must verify if all nodes have been visited.
If all nodes are visited, this indicates termination, therefore the termination flag must be activated.
This process can be formulated as \texttt{all-one(X, C, E) = write(all(X[2:,C]), X[1,E])}, where \texttt{all} is a function indicating whether all elements are non-zero, while $X[2\!:,C]$ captures the last $n$ rows of the column $C$.

We start by defining the attention stage:
\begin{align*}
(W^{(1)}_K, W^{(1)}_Q)_{i, j} = \begin{cases}
    1 &\text{if } (i,j)\in\left\{((P)_1, 1), ((P)_2, 2), (B_\text{global}, 2)\right\}\\
    0 &\text{otherwise,}
\end{cases}\quad
(W^{(1)}_V)_{i, j} = \begin{cases}
   -1 &\text{if } i=E,\; j=i\\
   1 &\text{if } i=B_\text{global},\;j=E\\
  0 &\text{otherwise.} 
\end{cases}
\end{align*}

The first attention head, defined as the read function in \Cref{sec:read_appdx}, is tasked with clearing fields used in the writing process. Additionally, we define:
\begin{align*}
(W^{(2)}_K)_{i, j} = \begin{cases}
    -1 &\text{if } i=B_\text{global}\\
    1 &\text{if } i=B_\text{local}\\
    0 &\text{otherwise,}
\end{cases}
\quad
(W^{(2)}_Q)_{i, j} = \begin{cases}
    1 &\text{if } i=B_\text{global}\\
    -1 &\text{if } i=B_\text{local}\\
    0 &\text{otherwise,}
\end{cases}
\quad
(W^{(2)}_V)_{i, j} = \begin{cases}
  \Omega &\text{if } i=C,\; j=E\\
  -\Omega &\text{if } i=B_\text{local},\; j=E\\
  0 &\text{otherwise.} 
\end{cases}
\end{align*}

This results in the following attention matrix:
 \begin{equation}
 \label{eq:attention_matrix_term}
  \small
\sigma\!\left(XW^{(2)}_QW^{(2)\top}_K X^\top\right) = { 
  \sigma\left(2\left[\begin{array}{c|cccccc} -1 & 1& \Compactcdots& 1 \\
  \hline 1 & -1& \Compactcdots & -1\\ 
  \vdots & \vdots & \ddots & \vdots \\ 
  1 & -1& \Compactcdots & -1\end{array}\right]\right)} 
  =
  {\left[\begin{array}{c|ccccc}
  0 & \nicefrac{1}{n}& \Compactcdots& \nicefrac{1}{n} \\
  \hline 1 & 0& \Compactcdots & 0 \\
  \vdots & \vdots & \ddots & \vdots \\ 
  1& 0& \Compactcdots &  0\end{array} \right]}  
\end{equation}

The top row of the attention matrix is crucial as it aggregates the last $n$ terms of the matrix $X$.
In addition, since our construction only uses columns with null values at the top row, the remaining $n$ rows in \eqref{eq:attention_matrix_term} do not add any contribution in any field.
We first define the terms for $f_\text{MLP}$, followed by an overview of how this implementation simulates the desired behavior.
\begin{align*}
(W^{(1)})_{i, j} &= \begin{cases}
    1 & \text{if } i=E,\, j\in\{E, S_{E}\}\\
    -1 & \text{if } i=E,\, j=S_{\Tilde{E}}\\
    0 & \text{otherwise,}
\end{cases}
\quad
(W^{(2, 3)})_{i, j} = \begin{cases}
    1 & \text{if } i\in\{E, S_E, S_{\Tilde{E}}\},\; j=i\\
    0 & \text{otherwise,}
\end{cases}\\
(W^{(4)})_{i, j} &= \begin{cases}
    1 & \text{if } i\in\{E, S_{\Tilde{E}}\},\; j=E\\
    -1 & \text{if } i=S_E,\; j=E\\
    0 & \text{otherwise.}
\end{cases}
\end{align*}

In the implementation of the $f_\text{MLP}$, the entries for $S_E$ and $S_{\Tilde{E}}$ are used to remove the entries pre-MLP, in both positive and negative cases.
In summary, the entire operation performs as follows: $\phi(1 -\Omega + \nicefrac{\Omega}{n}\sum_i(X[i,C]))$.
When the sum of elements in C is less than $n$, the term $-\Omega$ dominates, ensuring a resulting flag of zero.
However, once all nodes are visited, $X[:,C]$ offsets the negative influence of $\Omega$, thereby activating the flag.

\newpage
\subsection{Proof of \Cref{thm:dijkstra}}
\label{sec:dijkstra_appdx}
In this section, we introduce the constructions utilized to simulate Dijkstra's shortest path algorithm \cite{dijkstra1959note}.
This algorithm aims to determine the shortest path between any node of a weighted graph and a given source node.
To this end, we first present an overview of the implementation strategy, which accompanies the limitations that are described in \Cref{thm:dijkstra}.
These are then followed by the specific construction of each step of the algorithm.

\textbf{Implementation overview:} 
Executing Dijkstra's algorithm consists of an iterative process that detects the node with the current shortest distance from the source node,
updating the shortest path distances to neighboring nodes and marking the given node as visited. This process continues until all nodes have been visited.
In our implementation, we simulate a minimum function that operates in the column that stores the distances.
This operation is alternated with the main execution of the algorithm, that is, the section that reads the adjacency matrix and updates the distances.
This is made possible by the conditional selection mechanism and the minimum function's termination flag, which conditions any writing operations in the main execution of the algorithm.

\textbf{Specifications of the architecture:} 
In the following pseudocode description, we show the steps required to execute the algorithm.
Each of the steps has an associated transformer layer in the form of \eqref{eq:layer} that implements the corresponding routine.
In the case of Dijkstra's algorithm, there is a total of 17 steps, and therefore our implementation requires 17 layers.
Furthermore, throughout the implementation details of each step, no configuration uses parameters whose count scales with the size of the graph, thus accounting for constant network width.

\textbf{Limitations:} Beside the enumeration and the minimum edge limitations explained in \Cref{sec:theory}, the diameter of the graph also imposes a limitation on the set of graphs our constructions can represent.
Unlike the implementation of the other algorithms, the priority for selecting a node is determined by its current distance from the source node.
As explained in \Cref{sec:limitation_omega}, the constructions of conditional selection functions are limited by $\Omega$, the largest absolute value that can be utilized in a conditional clause.
This constructive limitation constrains the distances that can be represented.
More specifically, the greatest distance between any pair of vertices in the graph is bounded by $\Omega$.
Because of the reweighing scheme discussed in \Cref{sec:theory}, which guarantees that the smallest difference between any paths is lower bounded by $\varepsilon$, the reweighted edges will be greater than their original values.
Therefore, bounding the greatest distance of the reweighted graph is equivalent to bounding the diameter of the original graph to be $O(\Omega\epsilon)$.

In the remaining parts of the proof, we start by presenting the revised version of the algorithm,
followed by the construction of each step that has not been covered in previous algorithms.

{\centering
\begin{minipage}{.7\linewidth}
\begin{algorithm}[H]
  \small
  \caption{Dijkstra's algorithm for shortest path}
  \label{alg:Dijkstra}
\begin{algorithmic}[1]
  \REQUIRE {\bfseries Input:} integer start
  \REQUIRE {\bfseries Input:} matrix $A$, size $n \times n$
   \\\hrulefill
  \STATE prev, visit, dists, dists$_\text{masked}$, changes, is\_zero, candidates = arrays of size $n$
\FOR{$i = 1$ \TO $n$}
    \STATE visit[i], dists[i], prev[i] = \FALSE, $\hat{\Omega}$, i
\ENDFOR
  \STATE visit[start] = 0
  \STATE term = \FALSE
  \STATE $\cdots$
  \COMMENT{Initialization of min-variables}
   \\\hrulefill
  \WHILE{term is \FALSE}
    \FOR{$i = 1$ \TO $n$}
      \IF{visit[i] is \TRUE}
        \STATE dists$_\text{masked}$[i] = $\Omega$
        \COMMENT{(1) Mask visited nodes }[\ref{sec:mask_visited_appdx}]
      \ELSE
        \STATE dists$_\text{masked}$[i] = dists[i]
      \ENDIF
    \ENDFOR
   \\\hrulefill
    \STATE get\_minimum(dists$_\text{masked}$)
    \COMMENT{(2-8) Find minimum value }[\ref{sec:min_appdx}]
   \\\hrulefill
    \IF{term$_\text{min}$ is \TRUE}
    \STATE node = idx$_\text{best}$
    \STATE dist = val$_\text{best}$
    \COMMENT{(9) Get minimum values} [\ref{sec:update_var_min_appdx}]
    \ENDIF
    \STATE candidates = A[node, :]
    \COMMENT{(10) Get row of A }[\ref{sec:read_a_appdx}]
    \FOR{$i=1$ \TO $n$}
      \STATE is\_zero[i] = (candidates[i] $\le$ 0)
      \COMMENT{(11) Mark non-neighbors } [\ref{sec:mark_zeros_appdx}]
    \ENDFOR
    \FOR{$i=1$ \TO $n$}
        \STATE candidates[i] = candidates[i] + dist
        \COMMENT{(12) Build distances }[\ref{sec:repeat_appdx}]
    \ENDFOR
    \FOR{$i=1$ \TO $n$}
      \STATE changes[i] = candidates[i] $<$ dists[i]
      \COMMENT{(13) Identify updates }[\ref{sec:comparison_appdx}]
    \ENDFOR

    \FOR{$i=1$ \TO $n$}
        \IF{term$_\text{min}$ is \FALSE\text{ }\AND is\_zero[i] is \TRUE}
                \STATE changes[i] = 0
          \COMMENT{(14) Mask updates }[\ref{sec:mask_write_appdx}]
        \ENDIF
    \ENDFOR
  
    \FOR {$i = 1$ \TO $n$}
      \IF {changes[i] is \TRUE}
        \STATE prev[i], dists[i] = node, candidates[i]
        \COMMENT{(15) Update variables }[\ref{sec:update_var_appdx}]
      \ENDIF
    \ENDFOR
    \STATE visit[node] = visit[node] + term$_\text{min}$
    \COMMENT{(16) Visit node }[\ref{sec:visit_appdx}]
    \STATE term = \NOT (\FALSE \text{ in visit})
    \COMMENT{(17) Trigger termination }[\ref{sec:terminate_appdx}]
  \ENDWHILE
\REQUIRE \textbf{return} prev, dists
\end{algorithmic}
\end{algorithm}
\end{minipage}
\par
}
\subsubsection{Mask visited nodes: step (1)}
\label{sec:mask_visited_appdx}
In the execution of Dijkstra's algorithm, visited nodes are excluded from future iterations.
Our implementation achieves this by masking distances, ensuring that masked values are ignored during the minimum function execution.

Before delving into the parameter configuration of this layer, it is crucial to address the initialization of the \texttt{dists} variable, which stores current distances.
In Dijkstra and other algorithms, this variable is initialized to $\hat{\Omega} = \Omega - \epsilon$.
This setup ensures that, during the minimum function phase, there is a preference for unvisited nodes - that is, nodes not adjacent to any visited node at a certain point in the execution - over those already visited.
Conversely, initializing all nodes to $\Omega$ would result in the masking operation treating unvisited and visited nodes equally, thereby disrupting the accuracy of the simulation.

Implementing the masking operation in our context closely resembles the conditional selection function in \Cref{sec:if_else_appdx}, with a key difference in its initial setup. This variation primarily involves the definition of the first two layers, which we define as follows:
\begin{equation*}
(W^{(1)})_{i, j} = \begin{cases}
   1 & \text{if } i\in\left\{\substack{V_0, C, E,\\ B_\text{global}, B_\text{local}}\right\}, j=i\\
   \Omega & \text{if } i=B_\text{local}, j=S_{V_1}\\
   0 & \text{otherwise,}
\end{cases}
\quad
(W^{(2)})_{i, j} = \begin{cases}
   1 & \text{if } (i,j)\in \left\{\substack{(E, E), (V_0, S_1)\\(S_{V_1}, S_2)}\right\}\\
   -\Omega & \text{if } (i,j)\in \left\{\substack{(C,\, S_1), (B_\text{global},\, S_2),\\ (B_\text{local},\, S_2)}\right\}\\
   \Omega & \text{if } i = C, j = S_2\\
   0 & \text{otherwise,}\nonumber
\end{cases}\\
\end{equation*}

Unlike the approach outlined in \eqref{eq:if_else_appdx}, the masking operation does not have a $V_1$ counterpart already calculated.
To this end, we first generate create a mask of magnitude $\Omega$ over all nodes, and insert it in a scratchpad space denoted by $S_{V_1}$.
Then for the second layer, we simply redirect the definition presented in \eqref{eq:if_else_appdx}, from the dependency from $V_1$ to the dedicated column $S_{V_1}$.
These adjustments are processed by the remaining layers of the implementation, ensuring the correct nodes are selected for further operations in the algorithm.
For the application in step (1) of \Cref{alg:Dijkstra}, we set $V_0$ to \texttt{dists}, $C$ to \texttt{visit}, and E to \texttt{dists$_\text{masked}$}.

\subsubsection{Get minimum values: step (9)}
\label{sec:update_var_min_appdx}
The purpose of this function is to update the values \texttt{node} and \texttt{dist} once the minimum function terminates.
To this end, we also utilize the formulation of the conditional selection function, expressing it as \texttt{cond-select(X, [idx$_\text{best}$, val$_\text{best}$], [node, dist], term$_\text{min}$, [node, dist])}.
Note that in this formulation, the fields $V_0$, $V_1$, and $E$ take in multiple values.
The implementation for this case is similar to that presented in \Cref{sec:if_else_appdx}, matching the order of the elements in each of these fields.

\subsubsection{Get row of the adjacency matrix: step (10)}
\label{sec:read_a_appdx}
Once the variable \texttt{node} is updated, the corresponding row of the adjacency matrix $A$ must be extracted.
To this end, we refer to the implementation described in \Cref{sec:read_a}, setting $P_\text{cur}$ to \texttt{node}.

\subsubsection{Mark non-neighbor nodes: step (11)}
\label{sec:mark_zeros_appdx}
In this operation, all non-neighbor nodes must be indicated by 1 in their corresponding coordinates of the \texttt{is\_zero} variable. 
For this implementation, we set all attention parameters to zero, and define $f_\text{MLP}$ as follows:

\begin{align*}
(W^{(1)})_{i, j} &= \begin{cases}
    1 &\text{if } (i, j)=\{(B_\text{local}, S_\text{zero}), (\text{is\_zero}, \text{is\_zero})\}\\
    -\Omega &\text{if } i=\text{candidates},\;j=S_\text{zero}\\
    0 &\text{otherwise,}
\end{cases}
\quad
(W^{(2,3)})_{i, j} = \begin{cases}
    1 &\text{if } i=\{S_\text{zero},\text{is\_zero}\},\; j=i\\
    0 &\text{otherwise,}
\end{cases}\\
(W^{(4)})_{i, j} &= \begin{cases}
   1 &\text{if } i=S_\text{zero},\; j=\text{is\_zero}\\
   -1 &\text{if } i,j=\text{is\_zero}\\
    0 &\text{otherwise.}
\end{cases}
\end{align*}

In effect, we create a list of ones, which are then subtracted by all non-zero entries of \texttt{candidates}, multiplied by $-\Omega$.
This ensures that all non-neighbor nodes are set to one, while the remaining are zero.
In addition, the existing entries in the variable \texttt{is\_zero} are erased.

\subsubsection{Build distances: step (12)}
\label{sec:repeat_appdx}
In this step, the distance to the current node must be combined with the weighted edges.
To explain this process further, the operation copies a value from the top row across the bottom $n$ rows of X.
This operation can be formally expressed as \texttt{repeat-n(X, C, D) = write($\mathbf{1}_{n+1}$X[1,C], X[:,D])}. 
Here, $\mathbf{1}_{n+1}$ is the all-ones vector of size $n+1$, which is multiplied by the first entry of the column $C$. The result is then written in column $D$.

The first attention head is defined as:
\begin{align*}
(W^{(1)}_K)_{i, j} = \begin{cases}
    1 &\text{if } i=B_\text{global}\\
    0 &\text{otherwise,}
\end{cases}\quad
(W^{(1)}_Q)_{i, j} = \begin{cases}
    1 &\text{if } i\in \left\{B_\text{global}, B_\text{local}\right\}\\
    0 &\text{otherwise,}
\end{cases}\quad
(W^{(1)}_V)_{i, j} = \begin{cases}
  1 &\text{if } i=C,\; j=E\\
  0 &\text{otherwise.} 
\end{cases}
\end{align*}

This configuration leads to the following attention matrix:
 \begin{equation}
\label{eq:repeat_n}
 \small
  \sigma\!\left(XW^{(1)}_QW^{(1)\top}_K X^\top\right) = { 
  \sigma\left(2\left[\begin{array}{c|cccccc} 1 & 0& \Compactcdots& 0 \\
  \hline 1 & 0& \Compactcdots & 0\\ 
  \vdots & \vdots & \ddots & \vdots \\ 
  1 & 0& \Compactcdots & 0\end{array}\right]\right)} =
  {\left[\begin{array}{c|ccccc}
  1 & 0& \Compactcdots& 0\\
  \hline 1 & 0& \Compactcdots & 0 \\
  \vdots & \vdots & \ddots & \vdots \\ 
  1& 0& \Compactcdots &  0\end{array} \right]}.  
\end{equation}

Equation \eqref{eq:repeat_n} demonstrates that the operation on $X$ captures elements from the first row and distributes them across its subsequent rows
The second attention head is defined as follows:
 \begin{align}
    \label{eq:repeat_identity_appdx}
    (W^{(2)}_K, W^{(2)}_Q)_{i, j} = \begin{cases}
        1 &\text{if } (i,j)\in\left\{((P)_1, 1), ((P)_2, 2), (B_\text{global}, 2)\right\}\\
        0 &\text{otherwise,}
    \end{cases}\;
    (W^{(2)}_V)_{i, j} = \begin{cases}
       -1 &\text{if } i=C,E, j=E\\
      0 &\text{otherwise.} 
    \end{cases}
 \end{align}

The attention head in \eqref{eq:repeat_n} also generates unwanted values in the top row, which are eliminated by retrieving the column $C$ that contains the top row value replicated earlier.
As for the configuration of $f_\text{MLP}$, all its values are set to 0, only utilizing the residual connection.

For the specific application to step (12) in \Cref{alg:Dijkstra}, this is applied to \texttt{val$_\text{cur}$}
so that its value is distributed along all \texttt{candidates} entries in $X$.
However, since this field already contains the edges from the adjacent nodes, we set $(W^{(2)}_V)_\text{(val$_\text{cur}$, val$_\text{cur}$)}=0$, effectively ignoring the deletion of the target field.
Furthermore, notice that this operation builds the distance for all nodes, including both neighbors and non-neighbors.
The subsequent step in \Cref{alg:Dijkstra} addresses the non-neighbors by applying a masking technique.
Moreover, two additional variables are replicated: the minimum function's termination flag, utilizing \texttt{repeat-n(X, term$_\text{min}$, term$_\text{min\_n})$}, as well as the integer representation of the current node, through \texttt{repeat-n(X, node$_\text{int}$, node$_\text{int\_n}$)}.

The repeated integer representation of the node (\texttt{node$_\text{int\_n}$}) is employed in step (15) for updating the previously visited node.
On the other hand, the repeated termination flag (\texttt{term$_\text{min\_n}$}) is utilized in step (16) to mask operations if the minimum function is still in progress.

\subsubsection{Identify updates (13)}
\label{sec:comparison_appdx}
In this step, the distances computed in \texttt{candidates} are compared against the current distances. To this end, we utilize the implementation described in \Cref{sec:less_than} and write: \texttt{less-than(X, candidates, dists, changes)}.

\subsubsection{Mask updates: step (14)}
\label{sec:mask_write_appdx}
After performing the comparison, the variable \texttt{changes} indicates which nodes require updates. It is important to note that the comparison results should be set to false for non-neighbor nodes and when the termination flag of the minimum function is not active.
We express this condition as follows: \texttt{changes[i] = changes[i] and term$_\text{min}$ and (not is\_zero[i])}.

In terms of implementation, we begin by setting the parameters of $f_\text{attn}$ to zero and define the parameters of $f_\text{MLP}$ as follows:
\begin{align*}
(W^{(1)})_{i, j} &= \begin{cases}
    1 &\text{if } i, j\in\left\{\substack{(\text{changes, changes}), (\text{term$_\text{min\_n}$, changes}), \\(\text{changes, }S_\text{changes})}\right\}\\
    -1 &\text{if } i\in\left\{\substack{ \text{is\_zero},\\B_\text{global}, B_\text{local}}\right\},\; j=\text{changes}\\
    0 &\text{otherwise,}
\end{cases}\\
(W^{(2,3)})_{i, j} &= \begin{cases}
    1 &\text{if } i\in\{\text{changes}, S_\text{changes}\}, j=i\\
    0 &\text{otherwise,}
\end{cases}\quad
(W^{(4)})_{i, j} = \begin{cases}
   1 &\text{if } i,j=\text{changes}\\
   -1 &\text{if } i=S_\text{changes}, j=\text{changes}\\
    0 &\text{otherwise.}
\end{cases}
\end{align*}

In this implementation, we write the necessary logical condition as $\phi(\texttt{changes[i]} + \texttt{term$_\text{min}$} - \texttt{is\_zero[i]} - 1)$. 
This ensures the result equals 1 only when all conditions are satisfied.
Additionally, in this implementation, the variable \texttt{term$_\text{min\_n}$} is determined in step (12) by repeating the status of the flag \texttt{term$_\text{min}$} along the last $n$ rows of $X$.

\subsubsection{Update variables: step (15)}
\label{sec:update_var_appdx}
Once all the conditions have been imposed on the variable \texttt{changes}, the next step consists in updating the necessary values based on this decision variable.
The expression in line [38] of \Cref{alg:Dijkstra} is implemented as \texttt{cond-select(X, [node$_\text{int\_n}$, candidates], [prev, dists], changes, [prev, dists]}, utilizing the conditional selection function described in \Cref{sec:if_else_appdx}.
Here, the variable \texttt{node$_\text{int\_n}$} is the list containing the repeated entries of the current node in its integer form, as executed in step (12).

\subsubsection{Visit node: step (16)}
\label{sec:visit_appdx}
The principle for marking a node as visited follows the same implementation outlined in \Cref{sec:visit_min_appdx}.
We write this function as: \texttt{write-row(X, node, term$_\text{min}$, visit)}.
It is important to note that in the application of the \texttt{write-row} function, we assign the variable $D$ the variable \texttt{term$_\text{min}$}. This assignment is crucial because it ensures that the visit function is activated only when the minimum function's termination flag is also active.

\subsubsection{Terminate: step (17)}
\label{sec:terminate_appdx}
The termination function follows the implementation of \Cref{sec:terminate_min_appdx} by setting \texttt{all-one(X, visit, term)}.

\newpage
\subsection{Proof of \Cref{thm:bfs}}
\label{sec:bfs_appdx}
In this section, we introduce the constructions used to simulate the Breadth-First Search (BFS) algorithm \cite{moore1959shortest}.
The objective of this algorithm is to systematically explore nodes in a graph, starting from a root node and moving outward to adjacent nodes and beyond.
First, nodes directly adjacent to the root are visited, followed by their neighbors, and so on.
We will begin with an overview of our implementation strategy, including its limitations as detailed in \Cref{thm:bfs}
Next, we present the specific constructions of each step of the algorithm.

\textbf{Implementation overview:} 
In BFS, node traversal is typically managed using a queue \cite{cormen2022introduction}.
In this way, nodes nearer to the source node are discovered and queued for visitation first.
In this process, we also record their immediate predecessor, i.e., the node through which it was discovered.
This continues until all nodes connected to the source node have been visited.
Our implementation of BFS simulates a queue using a priority list combined with a minimum function, similar to the approach in \Cref{sec:dijkstra_appdx}.
Each time a node is selected by the minimum function, the priority value increases, thereby assigning lower priority to subsequent nodes discovered through it.
This strategy effectively emulates the first-in-first-out nature of a traditional queue.

\textbf{Specifications of the architecture:} 
We outline the algorithm's execution steps in the \Cref{alg:bfs}. Each step correlates with a transformer layer, as represented in Equation \eqref{eq:layer}, executing the respective routine.
In the case of the Breadth-First Search, there is a total of 17 steps, and therefore our implementation requires 17 layers.
Furthermore, throughout the implementation details of each step, no configuration uses parameters whose count scales with the size of the graph, ensuring constant network width.

\textbf{Limitations:} In addition to the enumeration limitation discussed in \Cref{sec:theory}, 
simulation is also constrained by the largest priority value that can be utilized.
A node is visited only once, and the priority value increments with each selection.
Hence, for a graph of size $n$, the priority can grow to a maximum of $n$ as well
As detailed in \Cref{sec:limitation_omega}, the conditional selection functions are limited by $\Omega$, the highest absolute value in a conditional clause.
This limitation restricts the maximum usable value and consequently the largest graph size that can be simulated. Specifically, based on conditional selection limits, the largest feasible graph is bounded by $O(\Omega)$

In the remaining parts of the proof, we start by presenting the revised version of the algorithm,
followed by the construction of each step that has not been covered in previous algorithms.

{\centering
\begin{minipage}{.7\linewidth}
\begin{algorithm}[H]
  \small
  \caption{Breadth-first search}
  \label{alg:bfs}
\begin{algorithmic}[1]
  \REQUIRE {\bfseries Input:} integer start
  \REQUIRE {\bfseries Input:} matrix $A$, size $n \times n$
   \\\hrulefill
  \STATE order = 0
  \STATE prev, visit, orders, orders$_\text{masked}$, changes, disc = arrays of size $n$
\FOR{$i = 1$ \TO $n$}
    \STATE visit[i], disc[i], orders[i], prev[i] = \FALSE, \FALSE, $\hat{\Omega}$, i
\ENDFOR
  \STATE visit[start] = order
  \STATE term = \FALSE
  \STATE $\cdots$
  \COMMENT{Initialization of min-variables}
   \\\hrulefill
  \WHILE{term is \FALSE}
    \FOR{$i = 1$ \TO $n$}
      \IF{visit[i] is \TRUE}
        \STATE orders$_\text{masked}$[i] = $\Omega$
        \COMMENT{(1) Mask visited nodes }[\ref{sec:mask_visited_appdx}]
      \ELSE
        \STATE orders$_\text{masked}$[i] = orders[i]
      \ENDIF
    \ENDFOR
   \\\hrulefill
    \STATE get\_minimum(orders$_\text{masked}$)
    \COMMENT{(2-8) Find minimum value }[\ref{sec:min_appdx}]
   \\\hrulefill
    \IF{term$_\text{min}$ is \TRUE}
    \STATE node = idx$_\text{best}$
    \COMMENT{(9) Get minimum value }[\ref{sec:update_var_min_appdx}]
    \ENDIF
    \STATE A$_\text{row}$ = A[node,:]
    \COMMENT{(10) Get row of A }[\ref{sec:read_a_appdx}]
    \STATE order = order + term$_\text{min}$
    \COMMENT{(11) Update priority factor }[\ref{sec:increase_priority_appdx}]

    \FOR {$i = 1$ \TO $n$}
    \STATE change = term$_\text{min}$ is \TRUE\text{ }\AND disc[i] is \FALSE\text{ }
    \STATE changes[i] = change is \TRUE \text{ }\AND A$_\text{row}$[i] is 1
    \COMMENT{(12) Identify updates }[\ref{sec:mask_changes_appdx}]
    \ENDFOR

    \FOR {$i = 1$ \TO $n$}
      \IF {changes[i] is \TRUE}
        \STATE prev[i] = node
        \COMMENT{(13) Update variables }[\ref{sec:update_var_appdx}]
        \STATE orders[i] = order
      \ENDIF
    \ENDFOR
    \STATE visit[node] = visit[node] + term$_\text{min}$
    \STATE disc[node] = disc[node] + term$_\text{min}$
    \FOR {$i = 1$ \TO $n$}
      \STATE disc[i] = disc[i] or changes[i]
    \COMMENT{(14) Visit and discover nodes }[\ref{sec:update_disc_appdx}]
    \ENDFOR
    \STATE interrupt = disc == visit
    \COMMENT{(15) Compare visited and discovered nodes }[\ref{sec:interrupt_appdx}]
    \STATE term = \NOT (\FALSE \text{ in visit})
    \COMMENT{(16) Trigger termination/interruption }[\ref{sec:term_bfs_appdx}]
    \STATE term = (term or interrupt) and term$_\text{min}$
    \COMMENT{(17) Trigger termination by interruption }[\ref{sec:inter_or_term_appdx}]    
  \ENDWHILE
\REQUIRE \textbf{return} prev

\end{algorithmic}
\end{algorithm}
\end{minipage}
\par
}

\subsubsection{Update priority factor}
\label{sec:increase_priority_appdx}
In this stage, as outlined in \Cref{alg:bfs}, the priority factor, denoted by \texttt{order}, needs to be increased.
This adjustment ensures that updated nodes are assigned a lower priority in the list.
It is important to note that \texttt{order} is only modified upon the activation of the minimum function's termination flag, which guarantees its update only when necessary.

Furthermore, the priority factor is utilized in the update step (13), so its value must be made available to all nodes.
This means that not only this value is increased, but also distributed for the local variables, in a variable we denote as \texttt{order$_\text{n}$}.

This implementation can be seen as a direct application of the \texttt{repeat-n} operation, detailed in \Cref{sec:repeat_appdx}, being expressed as \texttt{repeat-n(X, [term$_\text{min}$, order], [order$_\text{n}$, order$_\text{n}$])}.
Furthermore, to ensure that the updated value is kept for the next iterations we set $(W^{(2)}_V)_\text{(term$_\text{min}$, order)} = 1$ in the definition of \eqref{eq:repeat_identity_appdx}
Furthermore, as in the application of \Cref{sec:repeat_appdx} to Dijkstra's algorithm, the termination flag for the minimum function as well as the integer representation of the current node are also repeated along the last $n$ rows.

\subsubsection{Identify updates: Step (12)}
\label{sec:mask_changes_appdx}
In this step, we want to prevent any changes for non-neighbor nodes (i.e. \texttt{A$_\text{row}$[i]=0}), discovered nodes (i.e. \texttt{disc[i]=1}), or if the minimum function's termination flag (term$_\text{min}$) is deactivated.
This step closely follows the implementation of the masking function described in \Cref{sec:mask_write_appdx}.

However, in this case, we implement the following logical expression: \texttt{changes[i] = A$_\text{row}$[i] and term$_\text{min}$ and (not disc[i])}.
We define the parameters of $f_\text{MLP}$ as follows:
\begin{align*}
(W^{(1)})_{i, j} &= \begin{cases}
    1 &\text{if } i, j\in\left\{(\text{A$_\text{row}$, changes}), (\text{term$_\text{min\_n}$, changes}), (\text{changes, }S_\text{changes})\right\}\\
    -1 &\text{if } i\in\{B_\text{global}, B_\text{local}, \text{disc}\},\; j=\text{changes}\\
    0 &\text{otherwise,}
\end{cases}\\
(W^{(2,3)})_{i, j} &= \begin{cases}
    1 &\text{if } i\in\{\text{changes}, S_\text{changes}\}, j=i\\
    0 &\text{otherwise,}
\end{cases}\quad
(W^{(4)})_{i, j} = \begin{cases}
   1 &\text{if } i,j=\text{changes}\\
   -1 &\text{if } i=S_\text{changes}, j=\text{changes}\\
    0 &\text{otherwise.}
\end{cases}
\end{align*}

In this implementation, we write the logical condition as $\phi(\texttt{A$_\text{row}$[i]} + \texttt{term$_\text{min\_n}$[i]} - \texttt{disc[i]} - 1)$. 
This ensures the result equals 1 only when all conditions are satisfied.
Additionally, in this implementation, the variable
\texttt{term$_\text{min\_n}$} is determined in step (11) by repeating the status of the flag \texttt{term$_\text{min}$} along the last $n$ rows of $X$, as described in \Cref{sec:increase_priority_appdx}.

\subsubsection{Visit and discover nodes: step (14)}
\label{sec:update_disc_appdx}
In this step, both discovered and visited nodes are updated.
This implementation closely follows \Cref{sec:visit_min_appdx}, where we write \texttt{write-row(X, node, term$_\text{min}$, visit)} and \texttt{write-row(X, node, term$_\text{min}$, disc)}.
The divergence from the implementation of \Cref{sec:visit_min_appdx} is that, for the definition of $f_\text{MLP}$ we let:
\begin{equation*}
(W^{(1)})_{i, j} = \begin{cases}
    1 &\text{if } i=\text{changes}, j=\text{disc}\\
    0 &\text{otherwise,}
\end{cases}\quad
(W^{(2,3,4)})_{i, j} = \begin{cases}
    1 &\text{if } i,j=\text{disc}\\
    0 &\text{otherwise,}
\end{cases}
\end{equation*}
This extra modification ensures that all nodes that were changed in the previous step are also marked as discovered.
In addition, since the condition for change requires nodes not to have been discovered, this formulation does not overcount discovered nodes.

\subsubsection{Compare visited and discovered nodes: step (16)}
\label{sec:interrupt_appdx}
The principle of utilizing a priority list to determine which nodes to visit introduces a challenge in the simulation of the algorithm: the scenario when the graph is disconnected. 
As it stands, without any interruption mechanisms the implementation described in \Cref{alg:bfs} would erroneously visit nodes disconnected from the chosen root node, compromising the integrity of the simulation. 
To resolve this, we introduce an interruption criterion to identify when all nodes connected to the root have been visited.

This is achieved by comparing the list of visited nodes with the list of discovered nodes. 
After the first node is selected by the minimum function, a connected sub-graph is entirely visited if its list of visited nodes is the same as the discovered nodes.
In this step, we examine whether each element in the \texttt{visit} variable corresponds to its counterpart in the \texttt{disc} variable.
To this end, we set the attention parameters to zero and define the parameters of $f_\text{MLP}$ as follows:

\begin{align*}
(W^{(1)})_{i, j} &= \begin{cases}
   1 & \text{if } (i,j)\in\left\{\substack{(\text{visit}, S_1), (B_\text{local}, B_\text{local}),\\
   (\text{disc}, S_2), (\text{equal}, \text{equal})}\right\}\\
   -1 & \text{if } (i,j)\in\left\{(\text{visit}, S_2), (\text{disc}, S_1) \right\}\\
   0 & \text{otherwise,}
\end{cases}\;
&&(W^{(2)})_{i, j} = \begin{cases}
    1 & \text{if } (i,j)\in\left\{(B_\text{local}, S_1), (\text{equal}, \text{equal}) \right\}\\
   -1 & \text{if } i\in \left\{S_1, S_2\right\},\; j=S_1\\
   0 & \text{otherwise,}
\end{cases}\\
(W^{(3)})_{i, j} &= \begin{cases}
   1 &\text{if } i\in \{\text{equal}, S_1\},\; j=i\\
   0 & \text{otherwise,}
\end{cases}\quad
&&(W^{(4)})_{i, j} = \begin{cases}
   1 &\text{if } i=S_1,\; j=\text{equal}\\
   -1 &\text{if } i,j=\text{equal}\\
   0 & \text{otherwise.}
\end{cases}
\end{align*}

In this operation, we compute the subtraction of \texttt{visit} minus \texttt{disc} and also \texttt{disc} minus \texttt{visit}.
Ideally, if the two lists are identical, both subtractions should result in zero, leaving only the bias term introduced.
The outcome of these computations is then stored in a variable named \texttt{equal}, which we utilize in the next step of the operation.

\subsubsection{Compute termination and interruption criteria}
\label{sec:term_bfs_appdx}
The termination operation is described in the same form as \Cref{sec:terminate_min_appdx}, and expressed as: \texttt{all-one(X, visit, term)}.
A key addition to this calculation is the inclusion of the condition for the \texttt{equal} variable, computed in the previous step and described in \Cref{sec:interrupt_appdx}.
To streamline the process, we consolidate the comparison done in the previous step into a single variable, denoted \texttt{interrupt}.
Consequently, we write this expression as 
\texttt{all-one(X, equal, interrupt)}.

\subsubsection{Combine interruption and termination flags}
\label{sec:inter_or_term_appdx}
After calculating both the interruption and termination flags, the final step is to combine them using the following logic: \texttt{term = (term or interrupt) and term$_\text{min}$}.
The extra condition on the termination of the minimum function ensures that the simulation does not terminate in the first iteration, since in that stage, \texttt{visit} and \texttt{disc} list are both initialized with zeros.
We implement this condition by setting the parameters of $f_\text{attn}$ to zero, and we define the parameters of $f_\text{MLP}$ as follows:
\begin{align*}
(W^{(1)})_{i, j} &= \begin{cases}
   1 & \text{if } (i,j)\in\left\{\substack{(\text{term},\, \text{term}), (B_\text{global}, B_\text{global}),\\
   (\text{interrupt},\, S_1), (\text{term}_\text{min},\, \text{term}_\text{min})}\right\}\\
   -1 & \text{if } i=\text{term},\;j=S_1\\
   0 & \text{otherwise,}
\end{cases}\;
&&(W^{(2)})_{i, j} = \begin{cases}
      1 & \text{if } (i,j)\in\left\{\substack{(\text{term},\, \text{term}), (S_1, S_1),\\
   (\text{term},\, S_1), (\text{term}_\text{min},\, S_1)}\right\}\\
   -1 & \text{if } i=B_\text{global},\;j=S_1\\
   0 & \text{otherwise,}
\end{cases}\\
(W^{(3)})_{i, j} &= \begin{cases}
    1 & \text{if } i\in\left\{\text{term}, S_1\right\},\;j=i\\
   0 & \text{otherwise,}
\end{cases}\quad
&&(W^{(4)})_{i, j} = \begin{cases}
   1 &\text{if } = S_1,\; j=\text{term}\\
   -1 &\text{if } i,j=\text{term}\\
   0 & \text{otherwise.}
\end{cases}
\end{align*}

This above construction essentially realizes the intended logical condition through the following operation: $\texttt{term} = \phi\left(\,\phi\left(\texttt{interrupt}-\texttt{term}\right)+\texttt{term} + \texttt{term}_\texttt{min} -1\right)$. 
Here, the inner ReLU function, when added to \texttt{term}, calculates the logical disjunction, and the sum of this result with the other terms forms the input for the second ReLU function, thereby computing the logical conjunction.

\newpage
\subsection{Proof of \Cref{thm:dfs}}
\label{sec:dfs_appdx}
This section outlines the constructions used for simulating the Depth-First Search (DFS) algorithm \cite{moore1959shortest}.
DFS traverses nodes from the most recently discovered node, aiming to visit all nodes by progressing from one branch to its end before moving to another branch.
Initially, we provide an overview of our DFS implementation approach, addressing its constraints as specified in \Cref{thm:dfs}. 
Next, we present the specific constructions of each step of the algorithm.

\textbf{Implementation overview:}
One alternative to avoid the recursive formulation of DFS is to use a stack guided by node discovery order \cite{cormen2022introduction}.
In this way, nodes are explored in the order they are found, recording their immediate predecessor.
This continues until all nodes connected to the source node have been visited.
Our implementation of DFS simulates a stack using a priority list combined with a minimum function, similar to the approach in \Cref{sec:bfs_appdx}.
The priority decreases with each node selection via the minimum function, prioritizing newly discovered nodes, thus replicating the last-in-first-out property of a stack.

\textbf{Specifications of the architecture:} 
We outline the algorithm's execution steps in the \Cref{alg:dfs}.
Each step correlates with a transformer layer, as represented in Equation \eqref{eq:layer}, executing the respective routine.
In the case of the Depth-First Search, there is a total of 15 steps, and therefore our implementation requires 15 layers.
Furthermore, throughout the implementation details of each step, no configuration uses parameters whose count scales with the size of the graph, ensuring constant network width.

\textbf{Limitations:} In addition to the enumeration limitation discussed in \Cref{sec:theory}, 
simulation is also constrained by the largest priority value that can be utilized.
A node is visited only once, and the priority value decreases with each selection.
Hence, for a graph of size $n$, the absolute priority value can grow to a maximum of $n$ as well
As detailed in \Cref{sec:limitation_omega}, the conditional selection functions are limited by $\Omega$, the highest absolute value in a conditional clause.
This limitation restricts the maximum usable value and consequently the largest graph size that can be simulated. Specifically, based on conditional selection limits, the largest feasible graph is bounded by $O(\Omega)$

In the remaining parts of the proof, we start by presenting the revised version of the algorithm,
followed by the construction of each step, particularly those not covered in earlier algorithms.

{\centering
\begin{minipage}{.65\linewidth}
\begin{algorithm}[H]
  \small
  \caption{Depth-first search}
  \label{alg:dfs}
\begin{algorithmic}[1]
  \REQUIRE {\bfseries Input:} integer start
  \REQUIRE {\bfseries Input:} matrix $A$, size $n \times n$
   \\\hrulefill
  \STATE order = 0
  \STATE prev, visit, orders, orders$_\text{masked}$, changes = arrays of size $n$
\FOR{$i = 1$ \TO $n$}
    \STATE visit[i], orders[i], prev[i] = \FALSE, $\hat{\Omega}$, i
\ENDFOR
  \STATE visit[start] = order
  \STATE term = \FALSE
  \STATE $\cdots$
  \COMMENT{Initialization of min-variables}
   \\\hrulefill
  \WHILE{term is \FALSE}
    \FOR{$i = 1$ \TO $n$}
      \IF{visit[i] is \TRUE}
        \STATE orders$_\text{masked}$[i] = $\Omega$
        \COMMENT{(1) Mask visited nodes } [\ref{sec:mask_visited_appdx}]
      \ELSE
        \STATE orders$_\text{masked}$[i] = orders[i]
      \ENDIF
    \ENDFOR
   \\\hrulefill
    \STATE get\_minimum(orders$_\text{masked}$)
    \COMMENT{(2-8) Find minimum value }[\ref{sec:min_appdx}]
   \\\hrulefill
    \IF{term$_\text{min}$ is \TRUE}
    \STATE node = idx$_\text{best}$
    \COMMENT{(9) Get minimum value }[\ref{sec:update_var_min_appdx}]
    \ENDIF
    \STATE A$_\text{row}$ = A[node,:]
    \COMMENT{(10) Get row of A }[\ref{sec:read_a_appdx}]
    \STATE order = order - term$_\text{min}$
    \COMMENT{(11) Update priority factor }[\ref{sec:decrease_priority_appdx}]
    \STATE visit[node] = visit[node] + term$_\text{min}$
    \COMMENT{(12) Visit node }[\ref{sec:visit_appdx}]

    \FOR {$i = 1$ \TO $n$}
      \STATE change = term$_\text{min}$ is \TRUE\text{ }\AND visit[i] is \FALSE\text{ }
       \STATE changes[i] = change is \TRUE \text{ }\AND A$_\text{row}$[i] is 1
    \COMMENT{(13) Identify updates }[\ref{sec:mask_changes_dfs_appdx}]
    \ENDFOR

    \FOR {$i = 1$ \TO $n$}
      \IF {changes[i] is \TRUE}
        \STATE prev[i] = node
        \COMMENT{(14) Update variables }[\ref{sec:update_var_appdx}]
        \STATE orders[i] = order
      \ENDIF
    \ENDFOR

    \STATE term = \NOT (\FALSE \text{ in visit})
    \COMMENT{(15) Trigger termination } [\ref{sec:terminate_appdx}]    
  \ENDWHILE
\REQUIRE \textbf{return} prev

\end{algorithmic}
\end{algorithm}
\end{minipage}
\par
}

\subsubsection{Update priority factor: step (11)}
\label{sec:decrease_priority_appdx}
In this step, the priority factor, indicated by \texttt{order} in \Cref{alg:dfs}, is reduced.
This adjustment ensures that updated nodes receive a higher priority than the remaining in the priority list, effectively replicating the behavior of a stack.

The implementation of this routine is largely based on the construction detailed in \Cref{sec:increase_priority_appdx} used in the Breadth-First Search (BFS) algorithm, with a few key modifications. Specifically, we alter the values in the construction as follows: $(W^{(1)}V)\text{term$\text{min}$, order$\text{n}$} = -1$ and $(W^{(2)}V)\text{term$_\text{min}$, order} = -1$. These modifications are important as they enable the priority factor to decrease with each iteration of the main algorithm.

\subsubsection{Identify updates (13)}
\label{sec:mask_changes_dfs_appdx}
This step follows the same implementation of \Cref{sec:mask_changes_appdx}, replacing the variable \texttt{disc} with \texttt{visit}.

\newpage
\subsection{Proof of \Cref{thm:scc}}
\label{sec:scc_appdx}
In this section, we introduce the constructions utilized to simulate Kosaraju's Strongly Connected Components (SCC) algorithm \cite{aho1974design}.
This algorithm aims to detect the strongly connected components of a graph. A strongly connected component in a graph is a subgraph in which every vertex is reachable from every other vertex through directed paths.
To this end, we first present an overview of the implementation strategy, which accompanies the limitations that are described in \Cref{thm:scc}.
These are then followed by the specific construction of each step of the algorithm.

\textbf{Implementation overview:} Executing Kosaraju's SCC algorithm first consists of performing a depth-first search on the source node. To implement this, we create a priority list whose updated values get decreased at each iteration, effectively replicating the behavior of a stack.
For the current node in the execution, we check whether its neighbors have been visited.
If so, the corresponding value on the second priority list is updated.
This procedure simulates the insertion of nodes in a stack based on their finish times, as described in Kosaraju's SCC algorithm.

This process continues until each node has its values updated in the second priority list, which can take as much as $2n$ if the graph is a path.
The end of this process triggers a partial termination criterion, which activates the second phase of the algorithm.
In the second part, we perform another depth-first search. However, this time, the search is determined by the priority set in the previous phase, and it is performed on the transposed version of the graph, i.e. with the edges in the reverse direction.
The index to indicate a strongly connected component is the source node of the breadth-first search of each subgraph.
This process continues until all nodes have been visited, for a total of $3n$ node visits.

\textbf{Specifications of the architecture:} In the following pseudocode description, we show the steps required to execute the algorithm.
Each of the steps has an associated transformer layer in the form of \eqref{eq:layer} that implements the corresponding routine.
In the case of Kosaraju's Strongly Connected Components algorithm, there is a total of 22 steps, and therefore our implementation requires 22 layers.
Furthermore, throughout the implementation details of each step, no configuration uses parameters whose count scales with the size of the graph, thus accounting for constant network width.

\textbf{Limitations:} Beside the enumeration limitations explained in \Cref{sec:theory}, the number of node updates also determines the size of the graph that can be simulated.
This is because the update process is governed by a single variable (denoted \texttt{order} in \Cref{alg:scc}) that decreases every update.
As discussed in \Cref{sec:theory}, since the constructions are limited to a certain size in the input, one of the restrictions of simulation is that the graph size can have at most $\Omega/3$ nodes.

In the remaining of the proof, we start by presenting the revised version of the algorithm,
followed by the construction of each step that has not been covered in previous algorithms.
  
{\centering
\begin{minipage}{.7\linewidth}
\begin{algorithm}[H]
  \small
  \caption{Strongly Connected Components (SCC) Algorithm}
  \label{alg:scc}
\begin{algorithmic}[1]
  \REQUIRE {\bfseries Input:} integer start
  \REQUIRE {\bfseries Input:} matrix $A$, size $n \times n$
   \\\hrulefill
  \STATE term, term$_\text{part}$, order, node$_\text{ref}$ = \FALSE, \FALSE, 0, 0
  \STATE orders$_\text{masked}$, changes$_1$, changes$_2$, changes$_3$, nvisit = arrays of size $n$
   \STATE sccs = array of size $n$, init. [$1,\dots,n$]
  \STATE visit$_1$, visit$_2$, visit$_3$ = arrays of size $n$, init. \FALSE
  \STATE orders$_1$, orders$_2$ = arrays of size $n$, init. $\hat{\Omega}$

  \STATE orders$_1$[start], orders$_2$[start] = order, order
  \STATE $\cdots$
  \COMMENT{Initialization of min-variables}
   \\\hrulefill
  \WHILE{term is \FALSE}
    \STATE orders = orders$_2$ \textbf{if} term$_\text{part}$ is \TRUE \textbf{ else} orders$_1$
    \STATE visit = visit$_3$ \textbf{if} term$_\text{part}$ is \TRUE \textbf{ else} visit$_2$
    \COMMENT{(1) Select list to find minimum } [\ref{sec:select_min}]
    \FOR{$i = 1$ \TO $n$}
      \STATE orders$_\text{masked}$[i] = $\Omega$ \textbf{if} visit[i] is \TRUE \textbf{ else} orders[i]
      \COMMENT{(2) Mask visited nodes }[\ref{sec:mask_visited_appdx}]
    \ENDFOR
   \\\hrulefill
    \STATE get\_minimum(orders$_\text{masked}$, sccs)
    \COMMENT{(3-9) Find minimum value} [\ref{sec:min_appdx}]
   \\\hrulefill
   \STATE write$_1$ = term$_\text{min}$ \AND \NOT term$_\text{part}$
   \STATE write$_2$ = term$_\text{min}$ \AND term$_\text{part}$
      \COMMENT{(10) Update write flags }[\ref{sec:update_write_flags_appdx}]
    \IF{term$_\text{min}$ is \TRUE}
    \STATE node, scc = idx$_\text{best}$, scc$_\text{best}$
    \COMMENT{(11) Get minimum values }[\ref{sec:update_var_min_appdx}]
    \ENDIF
    \STATE A$_\text{row}$, A$_\text{col}$ = A[node,:], A[:, node]
    \COMMENT{(12) Get row/column of A } [\ref{sec:read_aT_appdx}] 
    \STATE visit$_1$[node] = visit[node] + write$_1$
    \STATE visit$_3$[node] = visit[node] + write$_2$
    \COMMENT{(13) Visit nodes }[\ref{sec:visit_scc_appdx}]
    \FOR {$i = 1$ \TO $n$}
      \STATE nvisit = (visit$_1$[i] \AND A$_\text{row}$[i] is 1)\text{ }\OR A$_\text{row}$[i] is 0
      \STATE nvisit[i] = nvisit\text{ }\AND write$_1$ is \TRUE
      \COMMENT{(14) Check visited neighbors } [\ref{sec:check_visit_neig_i_appdx}]
    \ENDFOR
    \STATE nvisit$_\text{all}$ = \NOT (\FALSE \text{ in nvisit})
    \COMMENT{(15) Check if all neighbors are visited } [\ref{sec:check_visit_neig_appdx}]
    \STATE visit$_\text{scc}$ = visit$_3$[scc]
    \COMMENT{(16) Check if SCC is visited }[\ref{sec:scc_visited_appdx}]
    \STATE node$_\text{ref}$ = scc \textbf{if} visit$_\text{scc}$ \textbf{else} node$_\text{ref}$
    \COMMENT{(17) Update reference node }[\ref{sec:if_else_appdx}]
    \FOR {$i = 1$ \TO $n$}
      \STATE changes$_1$[i] = write$_1$ \AND(A$_\text{row}$[i] is 1)\text{ }\AND \NOT visit$_1$[i]
      \STATE changes$_3$[i] = write$_2$ \AND(A$_\text{col}$[i] is 1)\text{ }\AND \NOT visit$_3$[i]
    \ENDFOR
    \STATE changes$_2$[node] = nvisit$_\text{all}$  \AND write$_1$
    \STATE order = order - term$_\text{min}$
    \COMMENT{(18) Identify updates and update priority } [\ref{sec:build_changes_scc_appdx}]
    
    \FOR {$i = 1$ \TO $n$}
       \STATE orders$_1$[i] = order \textbf{if} changes$_1$[i]\textbf{ else} orders$_1$[i]
        \COMMENT{(19) Update variables I }[\ref{sec:update_var_appdx}]
        \STATE orders$_2$[i] = order \textbf{if} changes$_2$[i]\textbf{ else} orders$_2$[i]
        \STATE visit$_2$[i] = visit$_1$[i] \textbf{if} changes$_2$[i]\textbf{ else} visit$_2$[i]
        \COMMENT{(20) Update variables II }[\ref{sec:update_var_appdx}]
        \STATE orders$_2$[i] = order \textbf{if} changes$_3$[i] \textbf{ else} orders$_2$[i]
        \STATE sccs[i] = node$_\text{ref}$ \textbf{if} changes$_3$[i]\textbf{ else} sccs[i]
        \COMMENT{(21) Update variables III }[\ref{sec:update_var_appdx}]
    \ENDFOR
    \STATE term = \NOT (\FALSE \text{ in visit$_3$})
    \STATE term$_\text{part}$ = \NOT (\FALSE \text{ in visit$_2$})
    \COMMENT{(22) Trigger (partial) termination} [\ref{sec:part_term_appdx}]    
  \ENDWHILE
\REQUIRE \textbf{return} sccs
\end{algorithmic}
\end{algorithm}
\end{minipage}
\par
}

\subsubsection{Select list to find minimum: step (1)}
\label{sec:select_min}
As mentioned in \Cref{sec:scc_appdx}, the implementation of the Strongly Connected Components Algorithm utilizes a two-staged approach consisting of two breadth-first searches.
As demonstrated in \Cref{sec:dfs_appdx}, the depth-first search is implemented using the minimum function and a priority list.
To enable the change between stages, we utilize a conditional selection function, as formalized in \Cref{sec:if_else_appdx} to select between the two priority lists and their corresponding lists of visited nodes for masking.
We express this function as \texttt{cond-select(X, [orders$_2$, visit$_3$], [orders, visit$_2$], term$_\text{part}$, [orders, visit])}.

\subsubsection{Update write flags: step (10)}
\label{sec:update_write_flags_appdx}
In this step, we utilize the partial termination criterion \texttt{term$_\text{part}$} to indicate which part of the function should be activated in the next steps of the execution.
To this end, we implement the logical conditions expressed in lines [15-16] of \Cref{alg:scc} as follows:
\begin{align*}
(W^{(1)})_{i, j} &= \begin{cases}
   1 & \text{if } i,j\in \{(\text{write$_1$, write$_1$}), (\text{write$_2$, write$_2$}),(\text{term$_\text{min}$}, S_1), (\text{term$_\text{min}$}, S_2), (\text{term$_\text{part}$}, S_2)\\
   -1 & \text{if } i,j\in \{(\text{term$_\text{part}$}, S_1), (B_\text{global}, S_2)\\
   0 & \text{otherwise,}
\end{cases}\\
(W^{(2,3)})_{i, j} &= \begin{cases}
    1 & \text{if } i\in\{\text{write$_1$, write$_2$}, S_1, S_2\}, j=i\\
   0 & \text{otherwise,}
\end{cases}\quad
(W^{(4)})_{i, j} = \begin{cases}
     1 & \text{if } i\in\{(S_1,\text{write$_1$}), (S_2, \text{write$_2$})\}\\
    -1 & \text{if } i\in\{\text{write$_1$, write$_2$}\}, j=i\\
   0 & \text{otherwise,}
\end{cases}
\end{align*}

We erase the original content from the \texttt{write$_1$} and \texttt{write$_2$} columns and insert $\text{write$_1$} = \phi(\text{term$_\text{min}$}-\text{term$_\text{part}$})$ and $\text{write$_2$} = \phi(\text{term$_\text{min}$}+\text{term$_\text{part}$}-1)$, respectively, which simulate the corresponding logical conditions.

\subsubsection{Read row and column of A: step (12)}
\label{sec:read_aT_appdx}
Once the node has been selected by the minimum function, we perform the extraction of both a row and a column of $A$.
The process of reading a row of $A$ is described in \Cref{sec:read_a}, and we can express it as \texttt{read-A(X, Ã, node, A$_\text{row}$)}
For the extraction of the column of $A$, we adopt the same procedure, utilizing an additional attention head that interacts with $A$.
The implementation details for this operation are similar to the configurations described in \Cref{sec:read_a}, which enables us to express it as \texttt{read-A(X, Ã.T, node, A$_\text{col}$)}.

\subsubsection{Visit nodes: step (13)}
\label{sec:visit_scc_appdx}
In addition to updating the list of visited nodes in the graph, the following step also serves to update an auxiliary variable, denoted \texttt{cur\_i}, which aims to store the position of the current node in a one-hot encoding. This can be easily accomplished by applying the implementation presented in \Cref{sec:visit_min_appdx} as follows: \texttt{write-row(X, node, write$_1$, cur\_i)}.

\subsubsection{Check visited neighbors: step (14)}
\label{sec:check_visit_neig_i_appdx}
Throughout the execution of the first depth-first search, we must check if the neighbors of the current node have been visited.
If this condition is met, the value of this node is updated in the second priority list, replicating the intended behavior of building a stack.
This is because nodes that meet this condition later in the algorithm receive higher priority in the second list. 
Additionally, this condition is also used to mask nodes from subsequent steps in the first DFS.

To this end, this process is implemented in two stages: steps (14) and (15).
In this step, we verify the condition expressed in lines [24-25] of \Cref{alg:scc}.
The parameters of $f_\text{attn}$ are set according to the implementation of \Cref{sec:repeat_appdx}.
More specifically, we implement the routines for \texttt{repeat-n(X, write$_1$, write$_{1_n}$)} and \texttt{repeat-n(X, write$_2$, write$_{2_n}$)}.
We then define the parameters of $f_\text{MLP}$ as follows:
\begin{align*}
(W^{(1)})_{i, j} &= \begin{cases}
   1 & \text{if } i,j\in \{(B_\text{global}, B_\text{global}), (\text{write$_{1_n}$}, \text{write$_{1_n}$}),\\
   &\hphantom{5em}(\text{nvisit}, \text{nvisit}), (\text{A$_\text{row}$}, S_1)\}\\
   -1 & \text{if } i=\text{visit$_1$},\; j=S_1\\
   0 & \text{otherwise,}
\end{cases}\,
&&(W^{(3)})_{i, j} = \begin{cases}
    1 & \text{if } i\in\{\text{nvisit}, S_1\},\;j=i\\
   0 & \text{otherwise,}
\end{cases}\\
\quad(W^{(2)})_{i, j} &= \begin{cases}
    1 & \text{if } i,j\in\{(\text{write$_{1_n}$}, S_1),(\text{nvisit, nvisit})\}\\
    -1 & \text{if } i,j=S_1\\
    -\Omega & \text{if } i=B_\text{global},\; j=S_1\\
   0 & \text{otherwise,}
\end{cases}
&&(W^{(4)})_{i, j} = \begin{cases}
     1 & \text{if } i=S_1,\;j=\text{nvisit}\\
    -1 & \text{if } i,j=\text{nvisit}\\
   0 & \text{otherwise.}
\end{cases}
\end{align*}

This implementation not only clears the target field \texttt{nvisit} but also simulates the logical conditions of lines [24-25] with the following expression $\phi(\text{write$_1$[i]}-\phi(\text{A$_\text{row}$[i]}-\text{visit$_1$[i]}))$.
The extra negative $\Omega$ used in $W^{(2)}$ is to remove any unwanted entries in the top row.

\subsubsection{Check neighboring condition: step (15)}
\label{sec:check_visit_neig_appdx}
In this step, we check if all neighbors of the current node have been visited. This information is then used to determine if the current node should be updated in the second priority list.
This implementation closely follows the one utilized for the termination criterion, since it also collects all values along the last $n$ rows of $X$. Therefore, we can express it as \texttt{all-one(X, nvisit, nvisit$_\text{all}$)}.

\subsubsection{Check if current SCC is visited: step (16)}
\label{sec:scc_visited_appdx}
In this stage, we must verify the condition of whether the Strongly Connected Component index that is assigned to the current node is visited.
This condition becomes important in the subsequent step.
Initially, the SCC index of each node is itself.
So, throughout the execution of the code, if we encounter an unvisited SCC index, it means that a new Strongly Connected Component has been found.
In this case, the \texttt{visit$_\text{scc}$} flag must be triggered so that in the next step, the reference node (\texttt{node$_\text{ref}$}) is updated.
To this end, we utilize the same formulation as the read function in \Cref{sec:read_appdx} and write \texttt{read-X(X, scc, visit$_3$, visit$_\text{scc}$)}.

\subsubsection{Identify updates and update priority: step (18)}
\label{sec:build_changes_scc_appdx}
This step of execution focuses on increasing the priority value and building the decision lists for the subsequent update steps.
Despite being represented as a global variable, in our implementation of \Cref{alg:scc}, we set \texttt{order} to be a local variable, since it is solely used for updating the values of the priority lists.
First, in the definition of $f_\text{attn}$, we set all parameters to zero. 
In the implementation of $f_\text{MLP}$, since we build three decision lists as well as the priority update, the parameters are presented separately for each of these purposes.

For the update of the priority values, we set:
\begin{align*}
(W^{(1)})_{i, j} &= \begin{cases}
   1 & \text{if } i\in\{\text{write$_{1_n}$}, \text{write$_{2_n}$}\}, j=S_\text{order}\\
   0 & \text{otherwise,}
\end{cases}\quad
(W^{(2, 3)})_{i, j} = \begin{cases}
    1 & \text{if } i=S_\text{order},\; j=i\\
   0 & \text{otherwise,}
\end{cases}\\
(W^{(4)})_{i, j} &= \begin{cases}
    -1 & \text{if } i=S_\text{order},\; j=\text{order}\\
   0 & \text{otherwise.}
\end{cases}
\end{align*}

In this part of the implementation, we simply subtract the writing flags that are distributed along the last $n$ rows from the order variable, thereby increasing the priority of newly updated nodes.

For building the first list of decision variables, we write:
\begin{align*}
(W^{(1)})_{i, j} &= \begin{cases}
    1 &\text{if } i\in\{(\text{A$_\text{row}$}, S_1), (\text{write$_{1_n}$}, S_1),(\text{changes$_1$},\text{changes$_1$})\}\\
    -1 &\text{if } i\in\{\text{visit$_1$}, 
    B_\text{local}\},\; j=S_1\\
    -\Omega &\text{if } i=B_\text{global},\; j=S_1\\
    0 & \text{otherwise,}
\end{cases}\\
(W^{(2, 3)})_{i, j} &= \begin{cases}
    1 &\text{if } i\in\{S_1,\text{changes$_1$}\},\; j=i\\
   0 & \text{otherwise,}
\end{cases}\quad
(W^{(4)})_{i, j} = \begin{cases}
    1 &\text{if } i=S_1,\; j=\text{changes$_1$}\\
    -1 &\text{if } i,j=\text{changes$_1$}\\
   0 &\text{otherwise,}
\end{cases}
\end{align*}

which simulates the condition in line 31 of \Cref{alg:scc} with the expression: $\phi(\text{A$_\text{row}$[i]}+\text{write$_{1_n}$[i]}-\text{visit$_1$[i]}-1)$.
The additional entry of $-\Omega$ serves to clear any undesired values in the top row.
For the second decision list we write:
\begin{align*}
(W^{(1)})_{i, j} &= \begin{cases}
    1 &\text{if } i\in\left\{(\text{cur\_i}, S_2), (\text{write$_{1_n}$}, S_2), (\text{nvisit}, S_2), (\text{changes$_2$}, \text{changes$_2$})\right\}\\
    -2 &\text{if } i\in\left\{B_\text{global}, B_\text{local},\right\},\; j=S_1\\
    0 & \text{otherwise,}
\end{cases}\\
(W^{(2, 3)})_{i, j} &= \begin{cases}
    1 &\text{if } i\in\{S_2,\text{changes$_2$}\},\; j=i\\
   0 & \text{otherwise,}
\end{cases}\quad
(W^{(4)})_{i, j} = \begin{cases}
    1 &\text{if } i=S_2,\; j=\text{changes$_2$}\\
    -1 &\text{if } i,j=\text{changes$_2$}\\
   0 &\text{otherwise,}
\end{cases}
\end{align*}

where \texttt{cur\_i} is the variable obtained in \Cref{sec:visit_scc_appdx} which represents the one-hot encoding of the current node.
This operation simulates the condition in line 34 by using the expression: $\phi(\text{write$_{1_n}$[i]}+\text{A$_\text{row}$[i]}+\text{cur\_i[i]}-2)$.

For the third decision list, we write:
\begin{align*}
(W^{(1)})_{i, j} &= \begin{cases}
    1 &\text{if } i\in\{(\text{A$_\text{col}$}, S_3), (\text{write$_{2_n}$}, S_3),(\text{changes$_3$},\text{changes$_3$})\}\\
    -1 &\text{if } i\in\{\text{visit$_1$}, B_\text{local}\},\; j=S_3\\
    -\Omega &\text{if } i=B_\text{global},\; j=S_3\\
    0 & \text{otherwise,}
\end{cases}\\
(W^{(2, 3)})_{i, j} &= \begin{cases}
    1 &\text{if } i\in\{S_3,\text{changes$_3$}\},\; j=i\\
   0 & \text{otherwise,}
\end{cases}\quad
(W^{(4)})_{i, j} = \begin{cases}
    1 &\text{if } i=S_3,\; j=\text{changes$_3$}\\
    -1 &\text{if } i,j=\text{changes$_3$}\\
   0 &\text{otherwise,}
\end{cases}
\end{align*}

This implementation simulates the condition in line 32  using the expression: $\phi(\text{A$_\text{col}$[i]}+\text{write$_{2_n}$[i]}-\text{visit$_3$[i]}-1)$.

\subsubsection{Trigger termination and partial termination: step (22)}
\label{sec:part_term_appdx}
The last step of the loop is to verify the termination criteria for the first and second phases of the algorithm.
Using the implementation described in \Cref{sec:terminate_min_appdx}, we simply write these expressions as: \texttt{all-one(X, visit$_3$, term)} and \texttt{all-one(X, visit$_2$, term$_\text{part}$)}.
Since these implementations utilize the same attention heads, their implementations are compatible and can be simply merged.

\newpage
\subsection{Proof of \Cref{remark:multitask}}
In this section, we present the methodology adopted for constructing a unified model capable of executing multiple graph algorithms, as stated in \Cref{remark:multitask}.
More specifically, we construct a multitask model that executes three distinct graph algorithms: Breadth-first search, Depth-first search, and Dijkstra's shortest path algorithm.
We limit our scope to these three algorithms, considering the complexity and minimal additional insights gained from incorporating more algorithms. 
However, this same design principle can be applied to a broader range of algorithms, including Strongly Connected Components (referenced as \Cref{sec:scc_appdx}).
Incorporating a new set of algorithms could potentially incur the introduction of new functions and a distinct set of variables to be integrated into the matrix $X$.

Some initial considerations are necessary to describe the construction strategy.
Each algorithm requires a distinct set of variables and functions.
In our constructions, this is reflected as specific columns in the input matrix and the different implemented layers, respectively.
Therefore, a model that unifies different algorithms must accommodate these individual components without compromising the execution of any particular algorithm.
Nevertheless, some algorithms can have common structures, sharing similar functions or variables utilized in analogous ways.
The challenge of multitasking extends beyond just encapsulating different executions within the same model. It also involves efficiently reusing shared variables and functions to avoid redundancy, which also significantly reduces overhead in terms of memory (number of columns of $X$) or runtime (number of layers).

Our goal is to provide a single implementation capable of executing one of these three algorithms given the appropriate input configuration.
The structure of the input is the same for all three algorithms.
However, its configurations slightly change for the execution of each algorithm.
While Breadth \& Depth-first search operate on unweighted graphs, the execution of Depth-first search is distinguished by the activation of a specific flag in $X$, denoted by $\gamma_s$.
For Breadth-first search and Dijkstra's algorithm, the configuration of $X$ is the same.
What sets Dijkstra's algorithm apart from BFS
is its operation on a weighted graph, as opposed to the unweighted graphs used by the other two algorithms.
This distinction also highlights the fact that Breadth-first search can be considered a special case of Dijkstra's algorithm when applied to unweighted graphs. Consequently, we can leverage a single execution for both breadth-first search and Dijkstra's algorithm.

Furthermore, all three algorithms share a large portion of similar functions.
For example, they all employ the minimum function during the initial phase of iteration and utilize a similar termination criterion.
Our implementation strategy consists of leveraging this shared structure while individually accommodating the unique functions of each algorithm.
The selection of specific elements necessary for executing a particular algorithm is managed by a selector function.
This function determines the variables that need to be updated, thus ensuring the execution reflects the intended algorithmic behavior.

The comprehensive structure of our implementation is illustrated in \Cref{alg:general}. Non-highlighted lines indicate the shared structural components common to all three algorithms. In contrast, colored lines denote algorithm-specific adaptations. Specifically, lines highlighted in blue (18, 22, 29, 38, and 39) are modifications for Depth-first search. Lines in red (24, 25, 30, 36, and 37) indicate the adaptations for both Dijkstra's and Breadth-first search. Lastly, the lines highlighted in orange (42-45) represent the conditional selection mechanism. This mechanism is crucial for dynamically selecting the algorithm-specific elements and the boolean variables that trigger these adaptations.

Since the implementations directly follow the specifications outlined in previous sections, the guarantees for each algorithm are established according to their respective designs (refer to \Cref{sec:bfs_appdx}, \Cref{sec:dfs_appdx}, and \Cref{sec:dijkstra_appdx}).
Furthermore, except for the selector function process, the details of each algorithmic step have been thoroughly discussed earlier.
In the following, we present the implementation of the selector function, along with a detailed description of the algorithm.

Additionally, we also conduct empirical validation, as detailed in \Cref{app:sec:empirical_validation}. This validation confirms the robustness of our unified implementation, described in \Cref{alg:general}, which demonstrates a 100\% accuracy across all tested instances of the three algorithms.

\subsubsection{Update priority factor: Step (12)}
\label{sec:cond_priority_factor}
As previously discussed in \Cref{sec:decrease_priority_appdx},
for the execution of Depth-first search, the priority variable \texttt{order} must be decreased at each iteration.
However, for the multitask model, this process should not be carried out if the model is executing a different algorithm.
To this end, we introduce a condition for updating the priority factor.

In our construction, we substitute the conditional form for an equivalent expression: $\text{order} = \text{order} - \phi(\text{term}_\text{min} + \gamma_s - 1)$.
This ensures that the variable \texttt{order} is only updated if \texttt{term$_\texttt{min}$} and $\gamma_s$ are activated.
We implement this condition by setting the parameters of $f_\text{attn}$ to zero, and we define the parameters of $f_\text{MLP}$ as follows:
\begin{align*}
(W^{(1)})_{i, j} &= \begin{cases}
   1 & \text{if } i\in\{\text{term}_\text{min}, \gamma_s\}, j=\text{order}\\
   -1 & \text{if } i=B_\text{global},j=\text{order}\\
   0 & \text{otherwise,}
\end{cases}\,
(W^{(2,3)})_{i, j} = \begin{cases}
   1 &\text{if } i,j=\text{order}\\
   0 & \text{otherwise.}
\end{cases}\\
(W^{(4)})_{i, j} &= \begin{cases}
   -1 &\text{if } i,j=\text{order}\\
   0 & \text{otherwise.}
\end{cases}
\end{align*}

The output of the last layer of $f_\text{MLP}$ is directly added to the residual connection X, effectively replicating the expression above.

\subsubsection{Select candidates and changes variables: Step (17)}
\label{sec:select_flag_appdx}
The variable \texttt{candidates} represents the values used for updating the current distances or priorities, essential for the Dijkstra/BFS and DFS algorithms. 
Specifically, \texttt{candidates$_1$} refers to the candidate values for Dijkstra/BFS, while \texttt{candidates$_2$} indicates those for DFS. Similarly, the \texttt{changes} variable is a boolean-flag array containing flags that indicate which values require updating.
The variables \texttt{changes$_1$} and \texttt{changes$_2$} correspond to the update flags for Dijkstra/BFS and DFS, respectively.

Finally, during the algorithm's execution, we must determine which variables are going to be chosen: \texttt{changes$_1$} and \texttt{candidates$_1$} or \texttt{changes$_2$} and \texttt{candidates$_2$}. This decision is guided by the boolean flag $\gamma_s$, which, when activated, indicates that the DFS routine should be executed, thereby selecting the second set of variables; otherwise, the first set is chosen. 
This expression is implemented as \texttt{cond-select(X, [change$_2$, candidates$_2$], [change$_1$, candidates$_1$], $\gamma_s$, [change, candidates]}, utilizing the conditional selection function described in \Cref{sec:if_else_appdx}.
Here, the variable $\gamma_s$ is also repeated along the last $n$ rows during step (14) of the algorithm, whose objective is to replicate the top row value along these rows.

\newpage
{\centering
\begin{minipage}{.7\linewidth}
\begin{algorithm}[H]
  \footnotesize
  \caption{General algorithm for DFS/BFS/Dijkstra}
  \label{alg:general}
\begin{algorithmic}[1]
  \REQUIRE {\bfseries Input:} integer start
  \REQUIRE {\bfseries Input:} \textcolor{orange}{bool $\gamma_s$, switch flag}
  \REQUIRE {\bfseries Input:} matrix $A$, size $n \times n$
   \\\hrulefill
   \STATE visit[start], order, term = 0, 0, \FALSE
  \STATE prev, visit, dists, dists$_\text{masked}$, changes, is\_zero, candidates = arrays of size $n$
\FOR{$i = 1$ \TO $n$}
    \STATE visit[i], dists[i], prev[i] = \FALSE, $\hat{\Omega}$, i
\ENDFOR
   
  \STATE $\cdots$
  \COMMENT{Initialization of min-variables}
   \\\hrulefill
  \WHILE{term is \FALSE}
    \FOR{$i = 1$ \TO $n$}
      \IF{visit[i] is \TRUE}
        \STATE dists$_\text{masked}$[i] = $\Omega$
        \COMMENT{(1) Mask visited nodes }[\ref{sec:mask_visited_appdx}]
      \ELSE
        \STATE dists$_\text{masked}$[i] = dists[i]
      \ENDIF
    \ENDFOR
   \\\hrulefill
    \STATE get\_minimum(dists$_\text{masked}$)
    \COMMENT{(2-8) Find minimum value }[\ref{sec:min_appdx}]
   \\\hrulefill
    \IF{term$_\text{min}$ is \TRUE}
    \STATE node = idx$_\text{best}$
    \STATE \textcolor{teal}{dist = val$_\text{best}$}
    \COMMENT{(9) Get minimum values} [\ref{sec:update_var_min_appdx}]
    \ENDIF
    \STATE A$_\text{row}$ = A[node, :]
    \COMMENT{(10) Get row of A }[\ref{sec:read_a_appdx}]
    \FOR{$i=1$ \TO $n$}
      \STATE \textcolor{teal}{is\_zero[i] = (A$_\text{row}$[i] $\le$ 0)}
      \COMMENT{(11) Mark non-neighbors } [\ref{sec:mark_zeros_appdx}]
    \ENDFOR
    
    \IF{\textcolor{purple}{$\gamma_s$ is \TRUE}}
        \STATE \textcolor{purple}{order = order - term$_\text{min}$}
    \COMMENT{(12) Update priority factor }
    [\ref{sec:cond_priority_factor}] 
    \ENDIF
    \STATE visit[node] = visit[node] + term$_\text{min}$
    \COMMENT{(13) Visit node }[\ref{sec:visit_appdx}]
    \FOR{$i=1$ \TO $n$}
        \STATE \textcolor{teal}{candidates$_1$[i] = A$_\text{row}$[i] + dist}\COMMENT{(14) Build candidates }[\ref{sec:repeat_appdx}]
        \STATE \textcolor{purple}{candidates$_2$[i] = order}
    \ENDFOR
    
    \FOR{$i=1$ \TO $n$}
      \STATE \textcolor{teal}{changes$_1$[i] = candidates$_1$[i] $<$ dists[i]}
      \COMMENT{(15) Identify updates }
      [\ref{sec:comparison_appdx}]
    \ENDFOR

    \FOR{$i=1$ \TO $n$}
    \STATE \textcolor{purple}{change$_2$ = term$_\text{min}$ is \TRUE\text{ }\AND visit[i] is \FALSE\text{ }} \COMMENT{(16) Build flags }[\ref{sec:mask_write_appdx}/\ref{sec:mask_changes_dfs_appdx}]
       \STATE \textcolor{purple}{changes$_2$[i] = change$_2$ is \TRUE \text{ }\AND A$_\text{row}$[i] is 1}
        \IF{\textcolor{teal}{term$_\text{min}$ is \FALSE\text{ }\AND is\_zero[i] is \TRUE}}
                \STATE \textcolor{teal}{changes$_1$[i] = 0}
        \ENDIF
    \ENDFOR
    
    \IF{\textcolor{orange}{$\gamma_s$ is \TRUE}}
        \STATE \textcolor{orange}{candidates, changes = candidates$_2$, changes$_2$}\COMMENT{(17) Select candidates/changes} [\ref{sec:select_flag_appdx}]
    \ELSE
        \STATE \textcolor{orange}{candidates, changes = candidates$_1$, changes$_1$}
    \ENDIF
    \FOR {$i = 1$ \TO $n$}
      \IF {changes[i] is \TRUE}
        \STATE prev[i], dists[i] = node, candidates[i]
        \COMMENT{(18) Update variables }[\ref{sec:update_var_appdx}]
      \ENDIF
    \ENDFOR

    \STATE term = \NOT (\FALSE \text{ in visit})
    \COMMENT{(19) Trigger termination }[\ref{sec:terminate_appdx}]
  \ENDWHILE
\REQUIRE \textbf{return} prev, dists
\end{algorithmic}
\end{algorithm}
\end{minipage}
\par
}
\newpage
\subsection{Proof of \Cref{remark:subleq}}
In this section, we present the results on the Turing Completeness of the architecture described in \eqref{eq:layer}.
Previous work by \citet{perez2021attention} and \citet{giannou23a} has established universality results for looped transformers with standard attention.
Nevertheless, it is important to assess how the expressiveness of the architecture is affected when graph connectivity is stored as a separate object and graph convolution is integrated into the attention mechanism.

To establish Turing Completeness, we adapt the method employed in \citet{giannou23a}.
In particular, we illustrate how our proposed architecture can efficiently emulate SUBLEQ, a single-instruction language recognized for its Turing Completeness \cite{mavaddat1988urisc}.

Our discussion begins with an overview of SUBLEQ, followed by the introduction of a modified version termed SUBLEQ$^-$ (read as \emph{SUBLEQ minus}), which is also Turing Complete.
Subsequently, we demonstrate that SUBLEQ$^-$ can be simulated using a SUBLEQ-like instruction that uses a specialized memory object for the adjacency matrix, which we name Graph-SUBLEQ.
Lastly, we establish that the architecture depicted in \eqref{eq:layer} is capable of simulating Graph-SUBLEQ.

\textbf{SUBLEQ:} Named for its operation ``subtract and branch if less than or equal to zero", SUBLEQ is a one-instruction set computer.
As detailed in \Cref{alg:subleq}, it consists of subtracting the content at address $a$ from that at address $b$, and storing the result back at $b$.
All these values are stored in a one-dimensional memory array.
If the result is non-negative, the computer executes the next instruction; otherwise, it jumps to the instruction at address $c$. Despite this operational simplicity, SUBLEQ is Turing Complete \cite{mavaddat1988urisc}.

{\centering
\vspace{-1em}
\begin{minipage}{.7\linewidth}
\begin{algorithm}[H]
    \small
  \caption{SUBLEQ$\,(a,b,c)$}
  \label{alg:subleq}
\begin{algorithmic}[1]
  \REQUIRE {\bfseries Input:} memory object $M$, addresses $a, b, c$
  \STATE $M[b] = M[b] - M[a]$
  \IF{$M[b]\le 0$}
        \STATE go to $c$
    \ELSE
        \STATE go to next instruction
    \ENDIF
\end{algorithmic}
\end{algorithm}
\end{minipage}
\par
}
\textbf{SUBLEQ$^-$:} In the remainder of the proof, we focus on a specialized variant of SUBLEQ, which we refer to as SUBLEQ$^-$.
This modified version operates similarly to the standard SUBLEQ, with a key distinction in handling the adjacency matrix in memory.
In SUBLEQ, a common approach to represent graph adjacency data involves vectorizing the adjacency matrix and placing it at the beginning of the memory.
In this case, specifically in SUBLEQ$^-$, the first $n^2$ memory entries, representing the row-major order vectorization of an adjacency matrix for a graph with $n$ nodes, are set to be read-only.
It is important to note that this read-only constraint on the first $n^2$ entries does not diminish the expressive power of SUBLEQ$^-$.
The data in these entries can always be accessed and then copied into writable memory locations using standard SUBLEQ instructions. 
Furthermore, these $n^2$ entries may not exist for routines that do not utilize the graph in the proposed way.
Consequently, even with this restriction, SUBLEQ$^-$ remains Turing Complete.

\textbf{Graph-SUBLEQ:} Building upon the concept of SUBLEQ$^-$, we introduce Graph-SUBLEQ, a new formulation with a distinct approach to memory management.
Unlike SUBLEQ$^-$, which uses a single memory object with varying writing permissions, Graph-SUBLEQ separates its memory into two distinct entities: a writable one-dimensional memory object, as in standard SUBLEQ, and a separate read-only memory dedicated to storing graph connectivity data.
Crucially, in this second memory, data is arranged in a matrix format, with the same structure as an adjacency matrix.

{\centering
\begin{minipage}{.7\linewidth}
\begin{algorithm}[H]
    \small
  \caption{Graph-SUBLEQ$\,(a,b,c, \gamma_a)$}
  \label{alg:graph_subleq}
\begin{algorithmic}[1]
  \REQUIRE {\bfseries Input:} memory objects $M$, $M_G$, addresses $a:=(a_1, a_2), b \text{ and } c$, flag $\gamma_a$
  \IF{$\gamma_a$ is \TRUE}
        \STATE $m_a = M_G[a_1, a_2]$
    \ELSE
        \STATE $m_a = M[a_1]$
  \ENDIF
  \STATE $m_b = M[b]$
  \STATE $M[b] = m_b - m_a$
  \IF{$M[b]\le 0$}
    \STATE go to $c$
  \ELSE
    \STATE go to next instruction
  \ENDIF
\end{algorithmic}
\end{algorithm}
\end{minipage}
\par
}
The implementation of Graph-SUBLEQ, as shown in \Cref{alg:graph_subleq}, closely resembles that of SUBLEQ in \Cref{alg:subleq}.
The main difference lies in the structure of the instructions. While SUBLEQ operates with the triplet $a,b,c$, a Graph-SUBLEQ instruction is defined by $a,b,c,\gamma_a$. In this context, $a$ represents a pair of addresses, $a_1$ and $a_2$.
These addresses serve two functions: together, they can access an element in the graph memory object $M_G$, or, using only $a_1$, access an element in the main memory $M$.
The decision of which memory is accessed by $a$ is determined by the boolean flag $\gamma_a$: if $\gamma_a$ is True, the instruction accesses the graph memory; otherwise, it accesses the main memory.
All other aspects of Graph-SUBLEQ's implementation are akin to those of SUBLEQ.

With the frameworks of SUBLEQ$^-$ and Graph-SUBLEQ established, we now present a lemma important for our concluding remark.

\begin{lemma}
Graph-SUBLEQ is Turing Complete.
\end{lemma}

\begin{proof}

To establish the Turing Completeness of Graph-SUBLEQ, we must demonstrate that every instruction in SUBLEQ$^-$ has an equivalent instruction in Graph-SUBLEQ.
This equivalence is essential to ensure that Graph-SUBLEQ can perform all operations that SUBLEQ$^-$ can.
The equivalence is examined in two scenarios based on the type of instruction in SUBLEQ$^-$: instructions that read from the read-only graph adjacency block, and those that read from other memory areas.
To maintain clarity in our notation, we use $a, b, c$ to represent instructions in SUBLEQ$^-$, and $\Tilde{a}_1, \Tilde{a}_2, \Tilde{b}, \Tilde{c}$ for instructions in Graph-SUBLEQ.

\textbf{No access to graph data:} To establish equivalence for instructions not accessing graph data, we define $\gamma_{\Tilde{a}} = 0$, and set $\Tilde{a}_1=a-n^2$, $\Tilde{a}_2=0$, $\Tilde{b}=b-n^2$, and $\Tilde{c}=c$. Here, $n$ represents the size of the graph.

\textbf{Access to graph data:} For instructions that access graph data, the equivalence is achieved by defining $\gamma_{\Tilde{a}} = 1$, and setting $\Tilde{a}_1=\lfloor a/n\rfloor$ and $\Tilde{a}_2=a-n\Tilde{a}_1$. We also set $\Tilde{b}=b-n^2$, and $\Tilde{c}=c$, where $n$ is the size of the graph.

When SUBLEQ$^-$ does not access graph data, the conversion to Graph-SUBLEQ requires a few adjustments to the memory addresses.
In Graph-SUBLEQ, the address $\Tilde{a}_2$ is not utilized in this scenario.
The addresses $\Tilde{a}$ and $\Tilde{b}$ in Graph-SUBLEQ are simply the corresponding SUBLEQ$^-$ addresses offset by $-n^2$, to accommodate the different indexing scheme in Graph-SUBLEQ's memory objects.
For the addresses $c$ and $\Tilde{c}$, no modification is needed, as they are already aligned and refer to elements outside the memory reserved for the graph.
In situations where a graph is not used in the memory, $n$ is equal to 0, making the addresses in SUBLEQ$^-$ and Graph-SUBLEQ naturally align.

In cases where SUBLEQ$^-$ accesses graph data through the address $a$, we derive the corresponding elements $\Tilde{a}_1$ and $\Tilde{a}_2$ in the matrix format of the graph memory.
Given that $a$ in SUBLEQ$^-$ represents an entry in the vectorized adjacency matrix arranged in row-major order, $\Tilde{a}_1$ is set to the integer part of $a/n$, and $\Tilde{a}_2$ is the remainder of the division $a/n$. This approach ensures proper alignment of indexes between SUBLEQ$^-$ and Graph-SUBLEQ in both scenarios,
thereby showing that Graph-SUBLEQ is also Turing Complete. 
\end{proof}

Having properly defined Graph-SUBLEQ and showing it that is Turing Complete, we now write the \Cref{remark:subleq} in its more precise form:

\begin{remark}
    There exists a looped-transformer $h_T$ in the form of \eqref{eq:layer}, which utilizes the modified attention head in \eqref{eq:attention_head}, with 11 layers, 3 attention heads, and layer width $O(1)$ that simulates Graph-SUBLEQ.
\end{remark}

\begin{proof}
In this proof, we employ the architecture outlined in Equation \eqref{eq:layer} to simulate \Cref{alg:graph_subleq}. 
This constructive approach is used to demonstrate that the architecture is Turing Complete.
As with our other proofs, we start by presenting an adaptation of \Cref{alg:graph_subleq} more aligned with our architecture.
Within this adapted algorithm, the commands executed inside the while loop correspond to the instructions of Graph-SUBLEQ.
Following the adaptation, we provide a comprehensive description of the input structure, detailing how the data and commands are organized and processed.

Finally, we delve into the systematic construction of the algorithm.
Each step has an associated transformer layer in the form of \eqref{eq:layer} that implements the corresponding routine.
In the simulation of Graph-SUBLEQ shown in \Cref{alg:graphleq_proof} there is a total of 11 steps, and therefore our implementation requires 11 layers, which use a total of 3 distinct attention heads.
Furthermore, throughout the implementation details of each step, no configuration uses parameters whose count scales with the size of the graph, thus resulting in constant network width.

{\centering
\begin{minipage}{.7\linewidth}
\begin{algorithm}[H]
    \small
  \caption{Graph-SUBLEQ$\,(a,b,c, \gamma_a)$ with external memory block}
  \label{alg:graphleq_proof}
\begin{algorithmic}[1]
  \REQUIRE {\bfseries Input:} memory objects $M$, $M_G$, instruction list $I$, instruction index $k$

\WHILE{\TRUE}
  \STATE $((a_1, a_2), b, c, \gamma_a) = I[k]$
  \COMMENT{(1) Read instructions } [\ref{sec:subleq_read}]
  \STATE $m_G = M_G[a_1, :]$
  \COMMENT{(2) Read row $a_1$ from $M_G$ }[\ref{sec:subleq_read_a}]
  \STATE $m_{a,G} = m_G[a_2]$
  \COMMENT{(3) Read address $a$ from $M_G$ }
  [\ref{sec:subleq_read}]
  \STATE $m_{a,X} = M[a_1]$
  \COMMENT{(4) Read address $a$ from $M$ } [\ref{sec:subleq_read}]
  \STATE $m_a = m_{a,G}$ \textbf{if} $\gamma_a$ is \TRUE \textbf{ else} $m_{a,X}$
  \COMMENT{(5) Select the memory value of $a$ } [\ref{sec:subleq_if_else}]
  \STATE $m_b = M[b]$
  \COMMENT{(6) Read address $b$ from $M$ }[\ref{sec:subleq_read}]
  \STATE diff $= m_b - m_a$
  \COMMENT{(7) Compute difference between values } [\ref{sec:subleq_subtraction}]
  \STATE $M[b] =$ diff
  \COMMENT{(8) Write difference in memory } [\ref{sec:subleq_write}]
  \STATE $k_\text{next} = k + 1$
  \COMMENT{(9) Compute next instruction } [\ref{sec:subleq_inc}]
  \STATE cond$_k$ = diff$\le 0$
  \COMMENT{(10) Compute condition for $k$ } [\ref{sec:subleq_less_than}]
  \STATE $k = c$ \textbf{if} cond$_k$ is \TRUE \textbf{ else} $k_\text{next}$
  \COMMENT{(11) Select next instruction } [\ref{sec:subleq_if_else}]
\ENDWHILE
\end{algorithmic}
\end{algorithm}
\end{minipage}
\par
}

\subsubsection{Input initialization}
We utilize the same format consistently used in this work, having an input matrix $X$ and an external adjacency matrix $A$.
The structure of the input matrix follows the convention outlined in \Cref{sec:input_matrix}, incorporating elements such as global and local variables, positional encodings, and biases. We maintain the same naming convention for consistency, with some adaptations specific to this context.
For instance, we also refer to the list of all positional encodings by $P$, while the positional encoding of the current instruction in this implementation is represented by $k$, as detailed in \Cref{alg:graphleq_proof}.

Given that in Graph-SUBLEQ, the size of the main memory and the graph might differ, directly translating these to the formats of $X$ and $A$ could result in a dimension mismatch, thereby hindering the multiplication process described in \eqref{eq:attention_head}.
To address this, we define $K=\max(|M|, n)+1$, where $K$ is the larger of the two dimensions, either the size of the main memory or the number of nodes in the graph.
This ensures compatibility in dimensions for encoding in $X$ and $A$.
In this setup, both the instructions and memory are padded with zeros in $X$ and $A$ to align with $K$.

In terms of specific variable representation, the main memory block, denoted as $M$, is represented as a single-column local variable within this framework. The instruction list, labeled as $I$, comprises a set of local variables, each representing an address in the instruction set.
These are individually identified by their specific assignments.
For example, $I_{a_1}$ refers to the column in the instruction list where addresses $a_1$ are stored.

\subsubsection{Read from input:  steps (1), (3), (4), and (6)}
\label{sec:subleq_read}
In the implementation of the read function within our architecture, we follow the construction outlined in \Cref{sec:read_appdx}.
For step (1), the operation is expressed as \texttt{read-X$\,\left(X,\,k,\, I_{a_1},\, Z_{a_1}\right)$}. Here, $k$ specifies the columns that store the positional encodings of the current instruction, and $Z_{a_1}$ is the variable holding the current value for $a_1$. The same formulation is then extended to all other values contained in the instruction, that is, $a_2, b, c$ and $\gamma_a$.

Moving to step (3), after the row of the graph memory is written into the variable $m_G$, this value is retrieved using the operation \texttt{read-X$\,\left(X,\,Z_{a_2},\,m_G,\,m_{a, G}\right)$}. In this context, $m_{a, G}$ is the target field where the retrieved value will be stored.

In step (4), the read operation is defined as \texttt{read-X$\,\left(X,\,Z_{a_1},\,M,\,m_{a, X}\right)$}. The variable $M$ represents the memory column, and $m_{a, X}$ is designated to hold the entry of $a$ retrieved from $M$.

Finally, for step (6), the operation is set as \texttt{read-X$\,\left(X,\,Z_b,\,M,\,m_b\right)$}. Here, $Z_b$ indicates the columns corresponding to the current address $b$, and $m_b$ is the variable designated to store the retrieved entry of $b$ from $M$.

\subsubsection{Read from $\Tilde{A}$: Step (2)}
\label{sec:subleq_read_a}
The implementation of step (2), which involves reading from the adjacency matrix $\Tilde{A}$, is based on the procedure detailed in \Cref{sec:read_a}. 
In this specific application, the read operation from $\Tilde{A}$ is described as \texttt{read-A$\,\left(X,\,\Tilde{A},, Z_{a_1},\, m_G\right)$}. Here, $Z_{a_1}$ represents the columns that contain the current value of $a_1$, which is used to locate the specific row in the adjacency matrix $\Tilde{A}$ that needs to be read. The variable $m_G$ is designated to store the retrieved row from $\Tilde{A}$.

\subsubsection{Conditional selection: steps (5) and (11)}
\label{sec:subleq_if_else}
The conditional selection process in steps (5) and (11) of our Graph-SUBLEQ implementation follows the approach outlined in previous algorithms, as detailed in \Cref{sec:if_else_appdx}.

In step (5), the conditional selection operation is expressed as \texttt{cond-select$\,\left(X,\, m_{a, G},\, m_{a, X}, Z_{\gamma_a},\, m_a\right)$}. Here, $m_a$ is the target field where the selected value will be stored. $Z_{\gamma_a}$ is the variable holding the current value of $\gamma_a$, which determines whether to select $m_{a, G}$ (the value from the graph memory) or $m_{a, X}$ (the value from the main memory).

For step (11), the operation is defined as \texttt{cond-select$\,\left(X,\, Z_c,\, k_\text{next}, \text{cond}_k,\, k\right)$}. In this context, $Z_c$ refers to the columns that hold the current value of $c$, and $k\text{next}$ represents the columns storing the index of the next instruction's positional encoding. The conditional logic applied here determines which instruction's positional encoding, either the current or next, is to be used based on the condition $\text{cond}_k$.

\subsubsection{Compute subtraction: step (7)}
\label{sec:subleq_subtraction}
In step (7) of our implementation, we focus on calculating the difference between the values stored in $m_b$ and $m_a$, and recording the result in a designated column, which we refer to as \texttt{diff}.
To carry out this operation, we set the parameters of $f_\text{attn}$ to zero
and define the parameters of $f_\text{MLP}$ as follows:
\begin{align*}
(W^{(1)})_{i, j} &= \begin{cases}
1 &\text{if } i\in\{(m_b, S_1), (\text{diff}, \text{diff})\}\\
-1 &\text{if } i\in\{(m_b, S_1), (\text{diff}, S_2)\}\\
0 &\text{otherwise,}
\end{cases}
\quad
(W^{(2, 3)})_{i, j} = \begin{cases}
1 &\text{if } i\in\{S_1, S_2, \text{diff}\},\;j=i\\
0 &\text{otherwise,}
\end{cases}
\\
(W^{(4)})_{i, j} &= \begin{cases}
1 &\text{if } i\in\{S_1, S_2\},\; j=\text{diff}\\
-1 &\text{if } i,j=\text{diff}\\
0 &\text{otherwise.}
\end{cases}
\end{align*}

Here, the first layers obtain the subtraction of the desired values, while the last layer places the result in the desired field. 
Additionally, the original value in \texttt{diff} is also replicated for both its positive and negative versions. These are used in the final layer to clear the previous entry.

\subsubsection{Write in memory: step (8)}
\label{sec:subleq_write}
For the implementation of step (8), where we write the result back into memory, we adopt the methodology outlined in  \Cref{sec:visit_min_appdx}.
We express this operation by \texttt{write-row$\,\left(X,\, Z_b,\, \text{diff},\, M\right)$}.
In this context, $X$ represents our input matrix, and $Z_b$ indicates the columns corresponding to the address $b$ in the memory, where the subtraction result (\texttt{diff}) needs to be written. The \texttt{diff} value holds the result of the subtraction computed in the previous step. The target for this write operation is the memory column $M$, which will be updated with the new value.

\subsubsection{Increment positional encoding: step (9)}
\label{sec:subleq_inc}
The implementation of step (9), which involves incrementing the positional encoding to move to the next instruction, follows the procedure described in \Cref{sec:increment_appdx}. 
The operation is defined as \texttt{increment$\,\left(X,\, k,\, k_\text{\normalfont next}\right)$}. Here, $X$ is the input matrix, $k$ represents the columns storing the current positional encoding, and $k_\text{\normalfont next}$ denotes the columns that will store the positional encoding of the next instruction. 

\subsubsection{Compute condition for next instruction: step (10)}
\label{sec:subleq_less_than}
In step (10), we need to check if the result of the subtraction (stored in \texttt{diff}) is less than or equal to zero. This is similar to the comparison function in Equation \Cref{eq:less_than}, but with a modification to include equality. Essentially, this function will return true if \texttt{diff} is zero or negative, allowing the algorithm to decide the next step based on this condition. 
In this case, we approximate the less-or-equal function as follows:

\begin{equation}
    \label{eq:le_subleq}
   X[:,C] \le X[:,D] \approx\: \varepsilon^{-1}\left(\phi\left(X[:,D]-X[:,C] + \varepsilon\right)-\phi\left(X[:,D]-X[:,C]\right)\right)
\end{equation}
where $C$ and $D$ are the columns used for comparison.
Notice that in this implementation, different from \eqref{eq:less_than}, the argument of the first ReLU includes an additional positive $\varepsilon$ value
Furthermore, for the application in the context of the simulation of Graph-SUBLEQ, the utilization of two columns $C$ and $D$ is not strictly necessary, since the difference between the values is already stored in the variable \texttt{diff} and the comparison is essentially made against zero.

To implement this function, we begin by setting the parameters $f_\text{attn}$ to zero. Then, we define the parameters of $f_\text{MLP}$ as follows:
\begin{align*}
(W^{(1)})_{i, j} &= \begin{cases}
1 &\text{if } i,j=\text{cond}_k\\
-1 &\text{if } i=\text{diff},\;j\in\{S_1, S_2\}\\
\varepsilon &\text{if } i=B_\text{global},\;j=S_2\\
0 &\text{otherwise,}
\end{cases}
\quad
&&(W^{(2)})_{i, j} = \begin{cases}
1 &\text{if } i,j=\text{cond}_k\\
\varepsilon^{-1} &\text{if } i,j=S_1\\
-\varepsilon^{-1} &\text{if } i=S_2,\;j=S_1\\
0 &\text{otherwise,}
\end{cases}\\
(W^{(3)})_{i, j} &= \begin{cases}
1 &\text{if } i\in\{\text{cond}_k, S_1\},\;j=i\\
0 &\text{otherwise,}
\end{cases}
\quad
&&(W^{(4)})_{i, j} = \begin{cases}
1 &\text{if } i=S_1,\;j=\text{cond}_k\\
-1 &\text{if } i,j=\text{cond}_k\\
0 &\text{otherwise.}
\end{cases}
\end{align*}

In the first layer of the process, we construct the arguments for the ReLU functions as defined in \eqref{eq:le_subleq}. Additionally, we maintain the value in \texttt{cond$_k$} to clear the target column.
Moving to the second layer, the ReLU terms are then processed by subtracting one from the other, and the resulting value is divided by $\varepsilon$.
In the final layer, this computed value is recorded in the target column, simultaneously erasing its original value.

\end{proof}

\end{document}